\documentclass[lettersize,journal]{IEEEtran}
\usepackage{amsmath,amsfonts}
\usepackage{array}
\usepackage{textcomp}
\usepackage{stfloats}
\usepackage{url}
\usepackage{verbatim}
\usepackage{cite}

\usepackage{amsthm}
\usepackage[colorlinks=True,
            linkcolor=blue,
            anchorcolor=green,
            citecolor=blue
            ]{hyperref} 
\usepackage[noabbrev,capitalise,nameinlink]{cleveref}
\usepackage{float}
\usepackage{graphicx}
\usepackage{epstopdf}
\usepackage{multirow}
\usepackage{caption,subcaption}
\usepackage{enumitem}

 \usepackage[ruled,vlined]{algorithm2e}
\usepackage{etoolbox}

\newtheorem{theorem}{Theorem}[section]

\newtheorem{lemma}{Lemma}[section]

\newtheorem{definition}{Definition}[section]
\newtheorem{remark}{Remark}[section]
\newtheorem{assumption}{Assumption}[section]

\newtheorem{setup}{Setup}[section]
\def\D{{\mathcal D}}
\def\K{{\mathcal K}} 
\def\S{{\mathcal S}}

\def\L{{\mathcal L}}
\def\R{{\mathbb  R}}
\def\P{{\mathbb  P}}
\def\E{{\mathbb  E}}

\def\ba{\mathbf{x}}

\def\bu{\mathbf{u}}
\def\bv{\mathbf{v}}

\def\bx{\mathbf{w}}

\def\bz{\mathbf{z}}
\def\tk{\tau_{k+1}}
\def\k{\tau_{k}}

\def\sgn{{\rm sgn}}

\def\med{{\rm med}}
\def\soft{{\rm soft}}

\def\tauk{\S^{\tau_{k+1}}}
\def\blx{\overline{\mathbf w}} 

\def\bxi{{\boldsymbol  \epsilon}}
\def\bze{{\boldsymbol  \zeta}}
\def\bpi{{\boldsymbol  \pi}}

\def\bg{{\boldsymbol  g}}

\def\FE{{\tt FedEPM}}

\def\SFA{{\tt SFedAvg}}

\def\FP{{\tt FedProx}} 
\def\SFP{{\tt SFedProx}} 
\def\eqspace{ \arraycolsep=1pt\def\arraystretch}
\graphicspath{ {./images/} }

\begin{document}

\title{Exact Penalty Method for Federated Learning}

\author{Shenglong Zhou and and Geoffrey Ye Li,~\IEEEmembership{Fellow, IEEE}
        % <-this % stops a space
\thanks{S.L. Zhou and G.Y. Li are with the ITP Lab,  Department of Electrical and Electronic Engineering, Imperial College London, London SW72AZ, United Kingdom,  e-mail: \{shenglong.zhou, geoffrey.li\}@imperial.ac.uk.}% <-this % stops a space
%\thanks{Manuscript received April 19, 2021; revised August 16, 2021.}
}



\maketitle

\begin{abstract}
Federated learning has burgeoned recently in machine learning,  giving rise to a variety of research topics.  Popular optimization algorithms are based on the frameworks of the (stochastic) gradient descent methods or the alternating direction method of multipliers.  In this paper, we deploy an exact penalty method to deal with federated learning and propose an algorithm, \texttt{FedEPM}, that enables to tackle four critical issues in federated learning: communication efficiency, computational complexity, stragglers' effect, and data privacy. Moreover, it is proven to be convergent and testified to have high numerical performance. 

\end{abstract}


\section{Introduction}\label{sec:introduction}
\IEEEPARstart{F}{ederated} learning (FL) \cite{konevcny2015federated,konevcny2016federated} is a recently cutting-edge technology in machine learning and has seen various applications in vehicular communications \cite{samarakoon2019distributed, pokhrel2020federated,
elbir2020federated, posner2021federated}, digital health 
\cite{rieke2020future}, mobile edge
computing \cite{mao2017survey, zhou2021communication}, just naming a few. However, it is still in its infancy and has many critical issues to be addressed \cite{kairouz2019advances,li2020federated,qin2021federated}, such as the communication efficiency, computation  efficiency, stragglers' effect, and data privacy. Before we present an overview of the relevant work, we would like to briefly introduce some mathematical background of   FL.

Generally speaking, FL is a collaborative approach that allows multiple (say $m$) clients (or devices) to train a shared model without exchanging their original data to maintain privacy. More specifically,  client $i$ has a loss function $f_i(\cdot) := f_i(\cdot;\D_i)$ associated with private data $\D_i$, where $f_i:\R^n{\to}\R$ is continuous and bounded from below. The task is to train a shared parameter (or model) $\bx^*$ by solving the following optimization problem,
\begin{eqnarray}\label{FL-opt}
 \arraycolsep=1.4pt\def\arraystretch{1.5}
\begin{array}{lll}
\bx^*={\rm argmin}_{\bx\in\R^n }~f(\bx):=\sum_{i=1}^m f_i(\bx).
\end{array}\end{eqnarray}
One of the popular approaches to address the above problem is based on the distributed optimization. The framework is depicted in Fig. \ref{fig:structure-FL}, where local clients update their parameters using the private data and then upload them to a central server for aggregation to get a shared parameter. However, such a framework induces a number of practical issues as follows. 
\begin{itemize}[leftmargin=17pt] 
\item[I1.]When exchanging parameters between clients and the server (i.e., at steps \textcircled{2} and \textcircled{4} in Fig. \ref{fig:structure-FL}), communication efficiency must be taken into consideration as frequent communications would consume expensive resources (e.g., transmission power, energy, and bandwidth). 
\item[I2.] Since there may be large numbers of clients engaging in the training, it is unrealistic to equip all of them with strongly computational devices. Hence, an advantageous FL algorithm should  reduce the computational complexity, thereby alleviating the computational burdens for clients. 
\item[I3.] Due to inadequate transmission resources and limited computational capacity, some clients might delay sharing parameters (which is known for the so-called stragglers' effect, namely, everyone waits for the slowest) or even withdraw from the training. To ensure a steady training process, this matter should be tackled. 
\item[I4.] As shown in Fig. \ref{fig:structure-FL}, clients send their updated parameters to the server at step \textcircled{4}. These parameters are possibly associated with their data directly, resulting in potential privacy disclosure during the training. So it is critical to preserve the privacy to eliminate the reluctance of clients before putting FL into practice.
\end{itemize}
In the subsequence, we present a brief overview of the relevant work based on the above four perspectives.
\begin{figure}[H]
	\centering
	\includegraphics[width=1\linewidth]{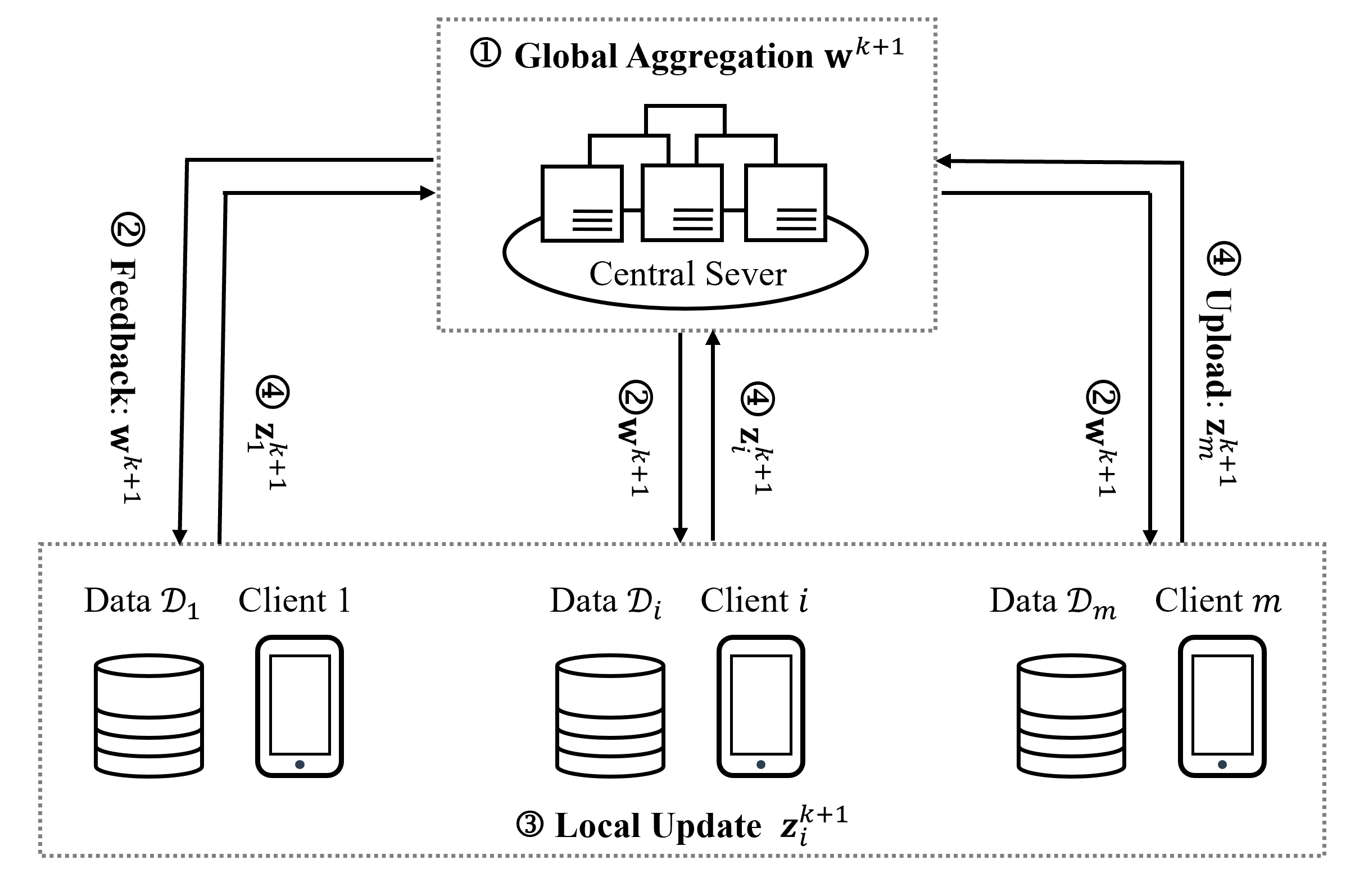}  
\caption{Structure of FL.\label{fig:structure-FL}}
\end{figure}
\subsection{Prior arts}
\subsubsection{Communication efficiency}\label{sec:pastwork-communication}
Data compression and communication rounds (CR) reduction are two popular techniques to improve communication efficiency. The former aims to quantize and sparsify the local parameters before the transmission so as to lessen the amount of the transmitted contents \cite{stich2018sparsified,di2018efficient,
wangni2018gradient,sattler2019robust}. For the latter, communications between clients and the server occur in a periodic fashion so as to reduce CR. More specifically, there is a tolerance of skipping steps \textcircled{1}, \textcircled{2}, and \textcircled{4} in Fig. \ref{fig:structure-FL} for several consecutive iterations, during which step \textcircled{3} is kept running. As a result, CR is reduced directly, which can improve  the communication efficiency significantly 
 \cite{mcmahan2017communication,li2019convergence,
 yu2019parallel,li2020federatedprox, zhou2022efficient}. In this paper, we will take advantage of this tactic.
\subsubsection{Computational complexity}\label{sec:pastwork-communication} In general, there is little effort for the server to aggregate all collected local parameters, so the major computations are imposed on clients who usually need to solve a series of sub-problems during the training. In order to diminish the computational complexity, there are two promising solutions. The first one is the stochastic approximation of some key items, such as the full gradients. However, it may take expensive prices (e.g., storage and time) to compute them, especially for complex functions and in big data settings. Then to fasten the computation, randomly selecting a small portion of data to approximate the full gradient (known as stochastic gradient) is effective. This idea has been extensively applied into the stochastic gradient descent (SGD) algorithms, such as the federated averaging ({\tt FedAvg} \cite{mcmahan2017communication}), local SGD  \cite{stich2018local, Lin2020Don}, and those in \cite{DeepLearning2015, AsynchronousStochastic2017, wang2021cooperative}. The second routine to reduce the computational complexity is solving sub-problems inexactly, which has been widely adopted in the inexact alternating direction method of multipliers (ADMM). They allow clients to update their parameters via solving sub-problems approximately, thereby alleviating the computational burdens and accelerating the learning speed exceptionally \cite{ding2019stochastic, Inexact-ADMM2021, ryu2022differentially, zhang2020fedpd}.

\subsubsection{Partial devices participation} To cope with the stragglers' effect, a straightforward solution is to avoid selecting clients who suffer from inadequate communication resources or limited computational capacity. In other words,  for each round of communication, the server is suggested to pick up a portion of clients in good conditions to engage in the training. This is often phrased as  partial devices participation  \cite{li2020federatedprox, zhou2022federated}. Such  schemes have  been assembled in  {\tt FedAvg} \cite{mcmahan2017communication}, secure federated averaging ({\tt SFedAvg} \cite{li2020secure}),  {\tt FedProx} \cite{li2020federatedprox},  {\tt FedADMM} \cite{zhou2022federated}, and {\tt  FedSPD-DP} \cite{li2022federated}. The former two algorithms were designed based on SGD while the latter two made full use of inexact ADMM.

\subsubsection{Privacy preservation}\label{sec:pastwork-privacy} 
 To maintain privacy, key elements (e.g., gradients) involving clients' data  are often perturbed by noises before sharing them with others. For instance, parameter $\bz_i^{k+1}$ in Fig \ref{fig:structure-FL} is usually formed  by a combination of some key items from local clients. Before sending $\bz_i^{k+1}$ to the server for aggregation, one of its key items should be contaminated by noises to prevent privacy leakage.  This idea is also known for differential privacy (DP) \cite{dwork2006calibrating,chaudhuri2011differentially,abadi2016deep}.  

An impressive body of work has developed optimization algorithms incorporating DP techniques to preserve privacy. We can summarize them into two groups. The first group targets FL problems from the primal perspective. For instance, differentially private SGD proposed in \cite{abadi2016deep} added the noise to perturb the stochastic gradient, and a similar strategy was also integrated in \cite{naseri2020local}. Moreover, some algorithms intended to contaminate the updated parameters directly, such as the secure federated averaging  \cite{li2020secure} and the noising before aggregation FL \cite{wei2020federated}.  The second group solves FL problems from the primal-dual perspective. Most of them were cast based on the frameworks of ADMM, where noises were added either on the primal variables (i.e., the trained parameters) directly \cite{guo2018practical,zhang2018improving, huang2019dp, wang2020privacy} or the dual variables (i.e., the Lagrange multipliers) \cite{zhang2016dynamic,zhang2018recycled}. In addition, the DP-based inexact ADMM in \cite{ryu2021differentially, ryu2022differentially}  inserted an noisy affine function  into the augmented Lagrange function.  In \cite{ding2019stochastic}, a differentially private stochastic ADMM has been cast and perturbed the gradient of the Lagrange function by noises.   Fairly recently, a federated stochastic primal-dual with DP algorithm was created in \cite{li2022federated}  and noises were added to perturb the combination of the primal and dual variables.

\subsection{Our contributions}
 
The main contribution of this paper is to develop a new FL algorithm with the following advantageous properties. 

\begin{itemize}[leftmargin=*] 
\item  Instead of solving original problem \eqref{FL-opt}, we focus on its penalized version (i.e. model \eqref{FL-EPM}) by using the elastic net regularization. Our established theory shows that the penalized model is exact for the original model. In other words, the solution to the original problem must be a solution to the penalized model when the penalized constant is large enough, as shown in Theorem \ref{the-EPM}. Therefore, it is rational to pay attention to the penalized model.

\item Differing from GD (or SGD) and ADMM-based algorithms, we benefit from the alternating direction method to solve the penalized model and cast a new FL algorithm, {\tt FedEPM}, which has a simple structure and is easy to be implemented. We then prove that the algorithm preserves the DP at each iteration and the generated sequence of the objective function values is convergent in expectation under very mild conditions that are weaker than those usually assumed for convergence analysis in FL.

\item The proposed algorithm is capable of dealing with four critical issues in FL. More specifically, since exchanging of parameters only occur at certain iterations, it is communication-efficient. In addition, sub-problems are solved inexactly and the main items (i.e., gradients) are computed only at certain iterations, resulting in a low computational complexity. Moreover, the algorithm exploits partial devices participation, so it is able to eliminate the stragglers' effect. Finally,  clients send the noisy parameters to the server for aggregation, thereby maintaining privacy.

\item Comparing with two leading FL algorithms, {\tt FedEPM} can achieve higher communication and computation efficiency as well as maintain stronger privacy.

\end{itemize}
\subsection{Organization}
This paper is organized as follows.  In the next section, we present all notations and introduce some useful functions.  In  Section \ref{sec:EPM},  we reformulate original problem \eqref{FL-opt} to a penalized version using the elastic net regularization and establish the exact penalty theory in Theorem \ref{the-EPM}. The algorithmic design and descriptions of advantageous properties of {\tt FedEPM}  are given in Section \ref{sec:ADMMFL}. We carry out the privacy and convergence analysis in Sections \ref{sec:privacy} and \ref{sec:converence}, respectively.  In Section  \ref{sec:num}, we conduct some numerical experiments to demonstrate the performance of our proposed algorithm. Some remarks are given in the last section.

%
%A function $f$ is said to be gradient Lipschitz continuous with a constant $r>0$ if for any $\bx$ and $\bz$,
% \begin{eqnarray}\label{Lip-r} 
%\|\nabla f(\bx)-\nabla f(\bz  ) \|  \leq r\| \bx -\bz  \|,
%  \end{eqnarray}
%where $\nabla  f(\bx)$ is the gradient of $f$ with respect to $\bx$. 

 \section{Preliminaries}\label{sec:pre}
 We begin with summarizing the notation   employed throughout this paper and then introduce several useful functions.
 
 \subsection{Notations}
 We use  plain,  bold, and capital letters to present scalars, vectors, and matrices, respectively. For instance, $\lambda$ and $\eta$ are scalars, $\bx$ and $\bv$  are vectors, and $W$ is a matrix. Let $\lfloor t\rfloor$ represent the largest integer smaller than $t{+}1$, e.g., $\lfloor 1.1\rfloor{=}\lfloor 2\rfloor{=}2$,  and $[m]{:=}\{1,2,\cdots ,m\}$ with `$:=$' meaning define.    We denote $\R^n$ the $n$-dimensional Euclidean space equipped with inner product ${\langle\cdot,\cdot\rangle}$ defined by $\langle\bx,\bv\rangle{:=}\sum_i w_iv_i$. Let $\|\cdot\|$ be the Euclidean norm (i.e., $\|\bx\|^2{:=}\langle\bx,\bx\rangle$) and $\|\cdot\|_1$ be the $1$-norm (i.e.,  $\|\bx\|_1{:=}\sum_i |w_i|$).  %Write the identity matrix as $I$ and a positive semidefinite matrix $A$ as $A\succeq 0$.   

\subsection{Some useful functions}
The 1-dimensional soft-thresholding operator is defined as
 \begin{eqnarray} \label{def-soft-1d} 
\eqspace{1.25} \begin{array}{lll}
\soft(t,a) &:=& {\rm arg}\min_{x}~ \frac{1}{2}(x-t)^2+ a  |x|\\
&=& \left\{\begin{array}{lccr}
t- a ,~~& t&>& a ,\\
0, & |t|&\leq& a ,\\
t+ a ,&t&<&- a .
\end{array}\right.
\end{array}
\end{eqnarray}
The sub-gradient of absolute function $|t|$ is given by
 \begin{eqnarray*}
 \eqspace{1.25}
\sgn(t) =\left\{\begin{array}{lll}
\{1\},& t>0,\\
{[-1,1]}, ~~& t=0,\\
\{-1\},&t<0.
\end{array}\right.
\end{eqnarray*}
We emphasize that the sub-gradient is a set rather than a scalar. Then the $n$-dimensional cases are defined elementwisely, that is, for  any $\bx\in\R^n$,
  \begin{eqnarray} \label{three-funcs}
   \eqspace{1.5}
 \begin{array}{lll}
 % \sign(\bx) &=& (\sign(x_1),\sign(x_2),\cdots,\sign(x_n))^\top,\\
\sgn(\bx) &=& (\sgn(w_1),\sgn(w_2),\cdots,\sgn(w_n))^\top,\\
\soft(\bx, a ) &=& (\soft(w_1, a ),\soft(w_2, a ),\cdots,\soft(w_n, a ))^\top\\
&=& {\rm argmin}_{\bz\in\R^n} \frac{1}{2}\|\bz-\bx\|^2+  a  \|\bz\|_1,\\
\end{array}\end{eqnarray}
where $\top$ stands for the transpose. Function  $f$ is said to be gradient Lipschitz continuous with constant $r{>}0$ if 
 \begin{eqnarray}\label{Lip-r} 
\|\nabla f(\bx)-\nabla f(\bv  ) \|  \leq r\| \bx -\bv  \|,~~\forall~\bx, \bv \in\R^n 
  \end{eqnarray}
where $\nabla  f(\bx)$  represents the gradient of $f$ with respect to $\bx$. Examples of  gradient Lipschitz continuous functions include the least squares and the logistic loss function. Hereafter, for a group  of vectors $\bx_i$  in $\R^n$, we denote 
\begin{eqnarray*}
\begin{array}{lll}
 W:=(\bx_1,\bx_2,\cdots ,\bx_m).
\end{array}\end{eqnarray*}
%Similar rules are also applied for $W^k, W^*, W^\infty$ and $\Pi^k, \Pi^*, \Pi^\infty$. Here $k, *$ and $\infty$ mean the iteration number, optimality and accumulation, e.g., see Corollary \ref{L-global-convergence}. 
The median of $W$ is calculated by  
  \begin{eqnarray*} 
\med(W) = 
\left[\begin{array}{c}
\med(w_{11},w_{21},\cdots,w_{m1})\\
\med(w_{12},w_{22},\cdots,w_{m2})\\
\vdots\\
\med(w_{1n},w_{2n},\cdots,w_{mn})
\end{array}
\right],
\end{eqnarray*}
where $w_{is}$ is the $s$th entry of $\bx_i$ and $\med(w_{1},w_{2},\cdots,w_{m})$ is the median value of $\{w_{1},w_{2},\cdots,w_{m}\}$. It is easy to see that
  \begin{eqnarray} \label{opt-med}
  \begin{array}{c}
\med(W) =  {\rm argmin}_{\bx\in\R^n} \sum_{i=1}^m \|\bx-\bx_i\|_1.\end{array}
\end{eqnarray}
%\subsection{ADMM}
%We now briefly introduce some backgrounds of the alternating direction method of multipliers (ADMM). For more details, one can refer to the earliest work \cite{gabay1976dual} and a nice book \cite{boyd2011distributed}.   
% Suppose we are given an optimization problem,
%\begin{eqnarray*} %\label{ADMM-opt-prob}
% %\arraycolsep=0pt\def\arraystretch{1.5}
%\begin{array}{cll}
%\min\limits_{\bx\in\R^n,\bz\in\R^q}~f (\bx)+  g(\bz),~
%{\rm s.t.}~~  A\bx+B\bz-\bb=0, 
%\end{array}\end{eqnarray*}
%where $A\in\R^{p\times n}$,  $B\in\R^{p\times q}$, and   $\bb\in\R^{p}$.
%To implement ADMM, we need the so-called augmented Lagrange function of  the above problem, which is defined by
%\begin{eqnarray*}%\label{ADMM-opt-ver11}
%\arraycolsep=1.4pt\def\arraystretch{1.5}
%\begin{array}{lll}
%&&\L(\bx,\bz ,\bpi) \\
%&:=& f (\bx)+  g(\bz) + \langle A\bx+B\bz-\bb, \bpi \rangle + \frac{\sigma}{2}\|A\bx+B\bz-\bb\|^2,
%\end{array}\end{eqnarray*}
%where $\bpi$ is known for the Lagrange multiplier and $\sigma>0$. Based on $\L$, ADMM  executes the following steps for a given initial point $(\bx^0, \bz^0,\bpi^0)$ iteratively 
%\begin{eqnarray}\label{framework-ADMM-0}
% \arraycolsep=1.4pt\def\arraystretch{1.5}
%\left\{\begin{array}{llll} 
% \bx^{k+1} &=&  {\rm argmin}_{\bx \in\R^n}~\L(\bx,\bz^k ,\bpi^k), \\ 
%  \bz^{k+1} &=& {\rm argmin}_{\bz\in\R^q}~\L(\bx^{k+1},\bz,\bpi^k), \\
%  \bpi^{k+1} &=&    \bpi^{k} + \sigma A\bx+B\bz-\bb.
%\end{array} \right.
%\end{eqnarray}
 
  \section{FL via Exact Penalty Method} \label{sec:EPM}
By introducing auxiliary variables, $\bx_i=\bx,i\in[m]$,  model \eqref{FL-opt} can be equivalently rewritten as
\begin{eqnarray}\label{FL-opt-ver1}
\begin{array}{lll}
 \underset{ \bx, W}{\min}~~  \sum_{i=1}^{m}   f_i(\bx_i),~~{\rm s.t.}~~\bx_i=\bx,~i\in[m].
\end{array}\end{eqnarray}
Instead of solving the above problem, we focus on the following penalty model, 
\begin{eqnarray}\label{FL-EPM}
\begin{array}{lll}
 \underset{ \bx, W}{\min}~~  F(\bx, W):=\sum_{i=1}^{m}  ( \underset{=:F_i(\bx, \bx_i)}{\underbrace{f_i(\bx_i) +  \varphi(\bx_i-\bx))}},
\end{array}\end{eqnarray}
where  $\varphi(\cdot)$ is the so-called elastic net regularization \cite{zou2005regularization},
\begin{eqnarray}\label{Elastic}
\begin{array}{lll}
 \varphi(\bz):=\lambda \|\bz\|_1 + \frac{\eta}{2} \|\bz\|^2,
\end{array}\end{eqnarray}
with $\lambda{>}0$ and $\eta{>}0$ are given constants.  Before embarking on solving the above problem, we present the optimality condition of problems  \eqref{FL-opt-ver1} and   (\ref{FL-EPM}) as follows.
\begin{definition}\label{def-sta} A point $(\bx^*, W^*)$ is a stationary point of   problem  (\ref{FL-opt-ver1}) if there are $\bpi^*_i,i\in[m]$ satisfying 
\begin{eqnarray}
 \arraycolsep=1.4pt\def\arraystretch{1.25}
\label{opt-con-FL-opt-ver1}
  \left\{\begin{array}{rcll}
 0 &=& \nabla  f_i(\bx_i^*)+\bpi_i^*, ~~&i\in[m],  \\ 
0&=& \bx_i^*-\bx^*
,&i\in[m],\\ 
0 &=& \sum_{i=1}^{m} \bpi_i^*.
\end{array} \right.
\end{eqnarray} 
A point $(\bx^o, W^o)$ is a stationary point of   problem  (\ref{FL-EPM}) if there are $\bpi_i^o\in \sgn(\bx_i^o-\bx^o), i\in[m]$ satisfying 
\begin{eqnarray}
 \eqspace{1.5}
\label{opt-con-FL-EPM}
  \left\{\begin{array}{rcll}
 0&=& \nabla  f_i(\bx_i^o)+ \lambda \bpi_i^o + \eta(\bx_i^o - \bx^o ) ,~~ i\in[m],  \\  
0&=&  \sum_{i=1}^{m} (\lambda \bpi_i^o + \eta(\bx_i^o - \bx^o ) ).
\end{array} \right.
\end{eqnarray} 
\end{definition}
It is not difficult to see that any optimal solution must be a stationary point for problems  \eqref{FL-opt-ver1} and   (\ref{FL-EPM}). If  $f_i$ is convex for every $i\in[m]$, then a point is an optimal solution if and only if it is a stationary point.    Based on the definition of stationary points, our first result shows that the penalty model is exact, namely,   a stationary point of problem  (\ref{FL-opt-ver1}) must be a stationary point of problem  (\ref{FL-EPM}) when $\lambda$ is over a threshold.
\begin{theorem}[Exact Penalty Theorem]\label{the-EPM} Let $(\bx^*, W^*)$ be a stationary point of    problem  (\ref{FL-opt-ver1}). Then it is also  a stationary point of   problem  (\ref{FL-EPM}) for any 
\begin{eqnarray}
 \eqspace{1.5}
\label{lower-bd-lambda}
 \begin{array}{rcll}
\lambda\geq \lambda^*:=   \max\limits_{i\in[m]}\max\limits_{j\in[n]} |(\nabla  f_i(\bx^*))_j|.
\end{array}\end{eqnarray}  
\end{theorem}
\begin{proof} Let $(\bx^o, W^o)=(\bx^*, W^*)$ and $\bpi_i^o = \bpi_i^*/\lambda$. Then \eqref{opt-con-FL-EPM} is ensured by \eqref{opt-con-FL-opt-ver1}. Therefore, we only need to verify $ \bpi_i^o\in \sgn(\bx_i^o-\bx^o), i\in[m]$. This can be derived by
\begin{eqnarray}
 \eqspace{1.5} 
 \begin{array}{rcll}
 \pi_{ij}^o &=& \pi_{ij}^*/\lambda ~\overset{\eqref{opt-con-FL-opt-ver1}}{=} -(\nabla  f_i(\bx^*))_j /\lambda\leq -\lambda^*/\lambda\\
 & \overset{\eqref{lower-bd-lambda}}{\in}& [-1,1] =  \sgn(0) \overset{\eqref{opt-con-FL-opt-ver1}}{=}  \sgn(w_{ij}^*-w^*_{j})\\
 &=& \sgn(w_{ij}^o-w_{j}^o),
\end{array}\end{eqnarray}  
 finishing the proof.
\end{proof}
The above theorem indicates that to find a stationary point of  original problem \eqref{FL-opt-ver1}, it makes sense to find  a stationary point of problem \eqref{FL-EPM} with a properly large $\lambda$. We next show a lemma beneficial to our algorithmic design.
\begin{lemma}\label{lemma-solution-EN} For given scalars $\{w_1,w_2,\cdots,w_m\}$, let $w_s^{\downarrow}$ be the $s$th largest scalar among them and define
\begin{eqnarray}\label{def-x-s}
 \eqspace{1.5} 
 \begin{array}{rcll}
h(w)&:=& \sum_{i=1}^m (\lambda|w-w_i| + \frac{\eta}{2}(w-w_i)^2),\\
w(s) &:=&  \frac{1}{m}\sum_{i=1}^m w_i + \frac{\lambda}{ \eta} (  \frac{2s}{m}-1)
\end{array}\end{eqnarray} 
If there is an $s^*$ satisfying $w_{s^*}^{\downarrow} {>} w(s^*) {>} w_{s^*+1}^{\downarrow}$, then
\begin{eqnarray}\label{x-s-*}
 \eqspace{1.5} 
 \begin{array}{rcll}
w^*=w(s^*) &=& {\rm argmin}_w h(w).
\end{array}\end{eqnarray} 
Otherwise, the solution can be derived by 
\begin{eqnarray}\label{x-s-*-w}
 \eqspace{1.5} 
 \begin{array}{rcll}
w^* = {\rm argmin}_{w\in\{w_1,\cdots,w_m\}} h(w).
\end{array}\end{eqnarray}
\end{lemma}
We note that solving \eqref{x-s-*-w} is quite easy since we can calculate $\{h(w_1),\cdots,h(w_m)\}$ and pick the smallest value. Moreover, second case \eqref{x-s-*-w} occurs in many scenarios, e.g., when $w_1=w_2=\cdots=w_m$.   
The above lemma is easily extended to multi-dimensional case.
\begin{lemma} For given vectors $\{\bx_1,\bx_2,\cdots,\bx_m\}$ in $\R^n$, let $w_{ij}$ be the $j$th entry of vector $\bx_i$ and
\begin{eqnarray}\label{x-s-*-wj}
 \eqspace{1.5} 
 \begin{array}{rcll}
w_j^*&:=&  {\rm argmin}_w \sum_{i=1}^m (\lambda|w-w_{ij}| + \frac{\eta}{2}(w-w_{ij})^2) 
\end{array}\end{eqnarray}
for any $j\in[n]$. Then it follows
\begin{eqnarray}\label{x-s-*-vec} \eqspace{1.5} 
 \begin{array}{rcll}
(w_1^*,w_2^*, \cdots,w_n^* )^\top= \begin{array}{rcll}
  {\rm argmin}_\bx \sum_{i=1}^m \varphi(\bx_i-\bx).
\end{array}
\end{array}\end{eqnarray} 
\end{lemma}
We present the implementation to calculate \eqref{x-s-*-vec} in Algorithm \ref{algorithm-Elasticnet}. We call it elastic net solver for given $m$ vectors $\{\bx_1,\bx_2,\cdots,\bx_m\}$, denoted as \texttt{ENS}$(\bx_1,\bx_2,\cdots,\bx_m)$. The computational complexity in the worse case is $O(nm\log(m))$.

 \begin{algorithm} [!th]
\SetAlgoLined
%{\noindent  Give $m$ vectors $\{\bx_1,\bx_2,\cdots,\bx_m\}$.}
 
\For{$j=1,2,\cdots, n$}{
{\noindent  Order  $\{w_{1j}, w_{2j},{\cdots},w_{mj}\}$ as $\{w_{1j}^{\downarrow}, w_{2j}^{\downarrow},{\cdots},w_{mj}^{\downarrow}\}$.}

\For{$s=1,2,\cdots, m$}{
Compute  $w_j(s) = \frac{1}{m}\sum_{i=1}^m w_{ij} + \frac{\lambda}{ \eta} (  \frac{2s}{m}-1)$

\If{$s<m$ and $w_{sj}^{\downarrow} > w_j(s) > w_{(s+1) j}^{\downarrow}$}{ Break the inner loop of $s$.

Return $w_j^*=w_j(s)$. 
}}

\If{$s==m$}{
Return $w_j^*$ by solving \eqref{x-s-*-wj}.
}
}
Return $(w_1^*,w_2^*, \cdots,w_n^* )^\top$.
\caption{\texttt{ENS}$(\bx_1,\bx_2,\cdots,\bx_m)$ \label{algorithm-Elasticnet}}

\end{algorithm}
 \section{Algorithmic Framework of \FE }\label{sec:ADMMFL}
To solve problem \eqref{FL-EPM}, we adopt the alternating direction method described as follows: Given a starting point $W^0$, perform the following iterations for $k=0,1,2,\ldots,$
 \begin{eqnarray}\label{FL-EPM-AD}
 \eqspace{1.5}
\begin{array}{lll}
\bx^{k+1}&\in&{\rm argmin}_{\bx} ~F(\bx, W^k)\\
&  =& {\rm argmin}_{\bx} \sum_{i=1}^{m}   \varphi(\bx_i^k-\bx), \\
\bx_i^{k+1}&\in&{\rm argmin}_{\bx_i} F_i(\bx^{k+1}, \bx_i),~~~~ i\in[m].
\end{array}\end{eqnarray}

The above scheme can be fitted into a FL setting where the server aggregates all parameters through the first sub-problem in \eqref{FL-EPM-AD} and client $i$ updates its parameter $\bx_i^{k+1}$ according to the second sub-problems in  \eqref{FL-EPM-AD}. However, this scheme will incur four critical issues (i.e., I1-I4 in Introduction \ref{sec:introduction}). To deal with these issues, we modify the alternating direction method to develop a new framework in Algorithm \ref{algorithm-ICEADMM}, where
\begin{eqnarray}
\eqspace{1.5}
\begin{array}{l}
\k :=\lfloor k/k_0 \rfloor, ~~\bg_i^{\tk}:= \nabla f_i( \bx^{\tk}). 
\end{array}\end{eqnarray}  
 Since the algorithm solves exact penalty model \eqref{FL-EPM}, we name it \FE. In the sequel, we would like to highlight the advantageous properties of \FE.
 \begin{algorithm} [!th]
\SetAlgoLined
{\noindent  Give an integer $k_0>0$, let $S^0=[m]$,  and set $k = 0$.} 
\For{Every client $i=1,2,\cdots,m $}{
{\noindent  Initializes  $\mu_{i,0}>0$, $c_{i}>0$, and $\alpha_i>1$.}
 
 {\noindent Generates a starting point $\bx_i^0$ and  a noisy vector $\bxi_i^0$.}
 
 {\noindent Uploads $\bz_i^{0}=\bx_i^0+\bxi_i^0$ to the server.} 
}

\For{$k=0,1,2,\cdots $}{

\If{$  k\in\K:=\{0,k_0,2k_0,3k_0,\cdots \}$}{
{\noindent  The  server randomly selects $\S^{\tk}\subseteq [m]$ and  broadcasts $\bx^{\tk} $ to clients in $\S^{\tk}$, where }
\begin{eqnarray}\label{iceadmm-sub1}
\eqspace{1.5}
 \begin{array}{llll}
\bx^{\tk} &=&  {\rm argmin}_{\bx} \sum_{i=1}^{m}   \varphi(\bz^{\k}_i-\bx).\\
&=&  \texttt{ENS}(\bz^{\k}_1,\bz^{\k}_2,\cdots,\bz^{\k}_m)
\end{array}
\end{eqnarray} 
% 
%}{
%\begin{eqnarray}\label{iceadmm-sub1-else}
% \begin{array}{llll}
%(S^{k+1},\bx^{k+1})  \equiv  (S^{k},\bx^{k}).
%\end{array}\end{eqnarray} 
}

\For{every client $i\in \S^{\tk}$}{
 
{\noindent Updates its parameters as } 
\begin{eqnarray}\label{iceadmm-sub2}
 \eqspace{1.75}
\begin{array}{lll} 
\mu_{i,k+1}&=& \mu_{i,0} (1+c_i\|\bx_i^k-\bx^{\tk}\|^2)\alpha_i^{k+1},\\
\widetilde{\bx}_i^{k+1}&:=&\mu_{i,k+1} (\bx_i^k-\bx^{\tk}) -\bg_i^{\tk},\\ 
\bx^{k+1}_i& = &\bx^{\tk}+\frac{\soft  (\widetilde{\bx}_i^{k+1},  \lambda )}{\eta+\mu_{i,k+1}}. 
\end{array} 
\end{eqnarray}
\If{$  k+1\in\K$}{
{\noindent Generates a randomly noisy vector $ \bxi_i^{\tk}$.}

{\noindent Sends $\bz^{\tk}_i$ to the server, where}
 \begin{align}
\label{iceadmm-sub4}  
\bz^{\tk}_i&=\bx_i^{k+1}  + \bxi_i^{\tk}. 
\end{align}  
}
}

\For{every $i\notin \S^{\tk}$}{ 

 Client $i$ keeps their parameters by  
\begin{eqnarray}
\label{iceadmm-sub5}
\eqspace{1.5}
\begin{array}{llll}   
 (\bx^{k+1}_i,\bz^{\tk}_i, \mu_{i,k+1} )  
 \equiv  (\bx^{k}_i,\bz^{\k }_i, \mu_{i,k}). 
\end{array}
\end{eqnarray}
}
}
\caption{{\FE }: FL by exact penalty method \label{algorithm-ICEADMM}}

\end{algorithm}
%\begin{figure}[!th]
%	\centering
%	\includegraphics[width=.99\linewidth]{FedRP.png}
%\caption{Structure of Algorithm \ref{algorithm-ICEADMM}.\label{fig:Structure}}
%\end{figure} 
\subsection{Communication efficiency} 
In Algorithm \ref{algorithm-ICEADMM}, clients and  the server communicate only  when 
 $k\in\K=\{0,k_0,2k_0,\cdots \}.$   Therefore, the larger $k_0$ the more steps for local updates and the fewer CR for convergence shown by our numerical experiments.  Therefore, this scheme is beneficial for improving communication efficiency.
 It is noted that such an idea has been extensively employed in literature \cite{AsynchronousStochastic2017,  yu2019parallel, wang2021cooperative,mcmahan2017communication,li2019convergence,
 li2020federatedprox, 
zhou2022federated, zhou2022efficient}.
\subsection{Computational complexity} 
One can observe that both the server and clients can solve their problems easily. More precisely, the sever solves \eqref{iceadmm-sub1} only at steps $k\in\K$, which  happens once for every consecutive $k_0$ iterations. The worst-case computational complexity  is $O(nm\log(m))$. At each iteration, clients update their parameters by \eqref{iceadmm-sub2}, where the major computation is to calculate $\bg_i^{\tk}{=}\nabla f_i( \bx^{\tk})$. Luckily, again for every consecutive $k_0$ iterations,  it only needs to be computed once due to $\k\equiv\tau_{sk_0},k=sk_0,sk_0+1,\cdots,sk_0+k_0$

We would like to emphasize that $\bx_i^{k+1}$  in \eqref{iceadmm-sub2} is an approximate solution to problem ${\rm min}_{\bx_i} F_i(\bx^{k+1}, \bx_i)$. In fact, it is a solution to the  problem,
\begin{eqnarray}\label{iceadmm-sub2-prob-1}
 \eqspace{1.5}
\begin{array}{r} 
   \bx^{k+1}_i = {\rm arg}\min\limits_{\bx_i}~ f_i(\bx^{\tk}) + \langle \bg_i^{\tk},\bx_i-\bx^{\tk}\rangle  \\
     +  \frac{\mu_{i,k+1 }}{2}\|\bx_i- \bx_i^k\|^2 + \varphi(\bx_i- \bx^{\tk}). 
\end{array} 
\end{eqnarray}
To see this, by letting $\bv_i  := \bx_i{-}\bx^{\tk}$, the above problem is equivalent to the following one,
\begin{eqnarray*} 
 \eqspace{1.5}
\begin{array}{lll} 
   \frac{\soft  (\widetilde{\bx}_i^{k+1},  \lambda )}{\eta+\mu_{i,k+1}}&=& {\rm arg}\min\limits_{\bv_i}~    \frac{1}{2}\Big\|\bv_i  -  \frac{\widetilde{\bx}_i^{k+1}}{\eta+\mu_{i,k+1}}\Big\|^2 {+}   \frac{\lambda}{\eta+\mu_{i,k+1}}\|\bv_i \|_1,
\end{array} 
\end{eqnarray*}
immediately resulting in the third equation in \eqref{iceadmm-sub2}.

In summary, \FE\ solves sup-problem inexactly and computes gradients only at certain iterations, thereby leading to relatively fast computation.  
\subsection{Partial devices participation}As partial clients are selected for the training at every step,  \FE \ enables us to deal with the straggler's effect. This strategy is similar to  {\tt FedAvg} \cite{mcmahan2017communication}, {\tt SFedAvg} \cite{li2020secure},  {\tt FedProx} \cite{li2020federatedprox},  {\tt FedADMM} \cite{zhou2022federated}, and {\tt  FedSPD-DP} \cite{li2022federated}. However, there is some difference among them.  First of all,  the global aggregation for the first three algorithms is employed on the selected clients in $\S^{\tau_k}$, namely,
\begin{eqnarray*} 
\eqspace{1.5}
 ~~\begin{array}{llll}
\bx^{\tk} =  {\rm arg}\min\limits_{\bx} \sum_{i\in\S^{\tau_k}}  \|\bz^{\k}_i-\bx\|^2 = \frac{1}{m} \sum_{i\in\S^{\tau_k}} \bz^{\k}_i,
\end{array}
\end{eqnarray*} 
where $\bz^{\k}_i=\bx^{k}_i+\bxi_i^{\k}$. 
The last two algorithms assemble parameters of all clients in $[m]$, namely,
 \begin{eqnarray*} 
\eqspace{1.5}
 \begin{array}{llll}
\bx^{\tk} =  {\rm arg}\min\limits_{\bx} \sum_{i=1}^{m}  \|\bv^{\k}_i-\bx\|^2 = \frac{1}{m} \sum_{i=1}^m \bv^{\k}_i,
\end{array}
\end{eqnarray*} 
where $\bv^{\k}_i$ is a combination of  $\bx^{k}_i$, the Lagrange multipliers, and noise $\bxi_i^{\k}$.  By contrast, although \FE\ aggregates parameters from all clients  as well, the aggregation is derived via solving sub-problem \eqref{iceadmm-sub1}, namely,
\begin{eqnarray*} 
\eqspace{1.5}
 \begin{array}{llll}
\bx^{\tk} =  {\rm argmin}_{\bx} \sum_{i=1}^{m}   \varphi(\bz^{\k}_i-\bx).\\ 
\end{array}
\end{eqnarray*} 
\subsection{Privacy preservation when communication}
We note that there is unnecessary to add noise to perturb its own parameter $\bx_i^{k+1}$ when the selected clients in $\S^{\tk}$ update their parameters. This enables us to ensure more accurate training. Only when they upload their parameters to the server (e.g., when $k+1\in\K$), the noise is added to perturb $\bx_i^{k+1}$ for the sake of protecting their data privacy. In the next section, we will show that such a strategy is capable of guaranteeing the so-called $\varepsilon$-differential privacy.

\section{Privacy Analysis}\label{sec:privacy}
 To establish the privacy guarantee, we introduce  the concept of $\varepsilon$-differential privacy \cite{chaudhuri2011differentially} as follows.
\begin{definition}[$\varepsilon$-Differential Privacy] A randomized mechanism $\cal M$ is $\varepsilon$-differentially private if for any two neighbouring datasets $\D$ and $\D'$
 differing in a single entry  and for any subsets of outputs $\cal O \subseteq {\rm range}(\cal M):$
  \begin{eqnarray}\label{def-DP}
  \begin{array}{lllll}
 \P({\cal M}(\D) \in {\cal O}) \leq  e^\varepsilon \cdot \P({\cal M}(\D')\in {\cal O}).
   \end{array}
  \end{eqnarray}  
 \end{definition}
 Similar to  \cite{ryu2022differentially}, we sample a noisy vector $\bxi$ with entries being independent and identically distributed (i.i.d.) Laplace variables with zero mean and a scale parameter  $\nu$, written as ${\rm Lap}(0,\nu)$. Its probability density function is given by
  \begin{eqnarray}\label{def-Laplace}
  \arraycolsep=1.4pt\def\arraystretch{1.5}
\begin{array}{lll}
d(\epsilon;0,\nu) = \frac{1}{2\nu} {\rm exp}\left\{-\frac{|\epsilon|}{2\nu}\right\}. 
\end{array}\end{eqnarray}
Particularly, for any $i\in[m]$, let $\widehat{\D}_i$ be a collection of datasets differing a single entry from  $\D_i$ and denote
\begin{eqnarray}\label{scale-para}
  \arraycolsep=0pt\def\arraystretch{1.5}
\begin{array}{lll}
 \Delta_i^{\tk}&:=&   \max\limits_{\D'_i \in \widehat{\D}_i } \|\bg^{\tk}_i(\D'_i)- \bg^{\tk}_i(\D_i)\|_1,\\
  \Delta_i^\infty &:=&      \max\limits_{{k+1}\in\K}\Delta_i^{\tk}.
\end{array}\end{eqnarray}  
\begin{setup}\label{setup-xi-1} For each $\tk\in\{0,1,2,\cdots\},$  every client $i\in\S^{\tk}$  generates $\bxi_i^{\tk}$ with entries $\{\epsilon_{ij}^{\tk}: j\in[n]\}$ being i.i.d. random variables from
\begin{eqnarray}\label{sample-noise}
  \arraycolsep=1.4pt\def\arraystretch{1.5}
\begin{array}{lll}
\epsilon_{ij}^{\tk}\sim {\rm Lap}\Big(0, \frac{ \Delta_i^{\tk}}{\varepsilon \mu_{i,k+1 } } \Big), ~~j\in[n],
\end{array}\end{eqnarray} 
 \end{setup}
Based on the above sampled noise,  \FE \ can ensure the $\varepsilon$-differential privacy at every $sk_0$-th iteration.
 \begin{theorem}\label{the-privacy} 
Under Setup \ref{setup-xi-1},  every $sk_0$-th iteration of \FE \ guarantees the $\varepsilon$-differential privacy, where $s=0,1,2\cdots$, namely,
 \begin{eqnarray}\label{the-privacy-eq}
 \arraycolsep=1.4pt\def\arraystretch{1.75}
  \begin{array}{lcl} 
 \P(\bz^{sk_0}|\D_i)  \leq  e^\varepsilon \cdot \P(\bz^{sk_0}|\D'_i). 
   \end{array}
  \end{eqnarray} 
 \end{theorem}

\section{Convergence analysis} \label{sec:converence}
This section aims to establish the convergence theory for \FE. We first present all assumptions on functions $f_i$. 
\begin{assumption}\label{ass-fi} Every $f_i, i\in[m]$ is gradient Lipschitz continuous with a constant $r_i>0$. 
\end{assumption} 
%\begin{assumption}\label{ass-f}  $f$ is coercive, i.e.,  $ \lim_{\|\bx\|\to\infty} f(\bx)=+\infty$.    
%\end{assumption}
\begin{setup}\label{setup-D-1} The server and clients adopt strategies as follows.
\begin{itemize}[leftmargin=*]
  \item The server randomly selects $\{\S^1,\S^2,\S^3,\cdots\}$ satisfying   
\begin{eqnarray}  \label{S-S-m}
   \eqspace{1.5}
   \begin{array}{lll}
  \S^{\tau+1} \cup\S^{\tau+2}\cup\cdots \cup\S^{\tau+s_0}=[m] 
   \end{array}
 \end{eqnarray}  
 for any $\tau{\in}\{0,s_0,2s_0,\cdots\}$, where $s_0{>}0$ is a given integer.
 \item  Every client $i{\in}[m]$ 
% initializes
% \begin{eqnarray}  \label{choice-mu0}
%   \eqspace{1.5}
%   \begin{array}{lll}
%   \mu_{i,0}>\max\{r_i-\eta, 4r_i^2/\eta\},    
%   \end{array}
% \end{eqnarray} 
%  and  
generates noisy vectors $\{\bxi_i^{0},\bxi_i^{1},\bxi_i^{2},\cdots\}$ as Setup \ref{setup-xi-1}. 
\end{itemize} 
\end{setup}
The selection of $S^\tau$ indicates that for each group of $s_0$  sets $\{\S^{\tau+1},\S^{\tau+2},\cdots,\S^{\tau+s_0}\}$, all clients should be chosen at least once. In other words,   the maximum gap between  two consecutive selections of any client  $i\in[m]$ is no more than $2s_0$, namely,
\begin{eqnarray}  
\label{scheme-omega-2T}\eqspace{1.25}
   \begin{array}{lll} \max \Bigg\{u-v: 
   \begin{array}{l}
   i\in\S^{v},  i\in\S^{u},\\ 
   i\notin (\S^{v+1}\cup  \S^{v+2} \cup \cdots \cup\S^{u-1}) 
   \end{array}
   \Bigg\} < 2s_0.
   \end{array}
 \end{eqnarray} 
\begin{remark} Condition (\ref{S-S-m}) can be satisfied with a high probability. In fact, if $\S^1, \S^2, \cdots$ are selected independently with $|\S^1|{=}|\S^2|{=}\cdots{=}\rho m$, where $\rho \in (0,1)$, and indices in every $\S^t$ are uniformly sampled from $[m]$ without replacement, then the probability of client $i$ being selected in $\{\S^{\tau+1},\S^{\tau+2},\cdots,\S^{\tau+s_0}\}$ is
\begin{eqnarray*} 
\arraycolsep=1.5pt\def\arraystretch{1.5}
\begin{array}{lrl}
p_i&:=&1-\P(i\notin\S^{\tau+1}, i\notin\S^{\tau+2},\cdots,i\notin\S^{\tau+s_0})\\
& =&1-\P(i\notin\S^{\tau+1})\P(i\notin\S^{\tau+2})\cdots\P(i\notin\S^{\tau+s_0})\\
&=&1-(1- \rho)^{s_0},
\end{array} \end{eqnarray*}  
which is quite close to $1$ when $s_0$ is large, e.g., $p_i=0.999$ if $s_0=10$ and $\rho=0.5$ for any $ \tau\geq 1$.  
\end{remark}
We now present some constants as follows,
\begin{eqnarray}\label{constants}
  \arraycolsep=0pt\def\arraystretch{2}
\begin{array}{lll}
 \L^{k} &:=& \E_{\bxi} F(\bx^{\k},W^{k}){+} \sum_{i=1}^{m}(\frac{r_i^2}{2\mu_{i,0}c_i (\alpha_i-1)\alpha_i^{k}}  {+}  2\phi_{i,k-1}  ),\\
  \phi_{i,k}&:=& \frac{4n\lambda\Delta_i^{\infty} \alpha_{i}^{2s_0k_0}}{\varepsilon \mu_{i,0}(\alpha_i-1) \alpha_{i}^{k}}  +  \frac{8n\eta(\Delta_i^\infty\alpha_{i}^{2s_0k_0})^2}{(\varepsilon \mu_{i,0})^2 (\alpha_i^2-1) \alpha_{i}^{2k  } }. 
\end{array}\end{eqnarray} 
Our first result shows the descent properties of sequences  $\{ \L^k\}$ and $\{F(\bx^{\k},W^{k})\}$.
\begin{lemma}\label{lemma-decreasing-1} Under Assumption \ref{ass-fi} , it holds
  \begin{eqnarray} \label{decreasing-property-2}  
   \eqspace{1.75}
  \begin{array}{lcl}
&& F(\bx^{\tk},W^{k+1})-  F(\bx^{\k },W^{k})\\ 
 &\leq& 
      \sum_{i=1}^{m}\Big(   \frac{r_i^2}{2\mu_{i,0}c_i\alpha_i^{k+1} }    + 2    \varphi( \bze_i^{\k } ) \Big)\\
      &-& \sum_{i=1}^{m}  \frac{\mu_{i,k+1 }+\eta-r_i }{2}  \|\bx_i^{k+1}-\bx_i^{k}\|^2\\
      & -& \sum_{i=1}^{m}\frac{\eta}{4}  \|\bx^{\tk}-\bx^{\k }\|^2, 
    \end{array}    
    \end{eqnarray} 
    where $\bze_i^{\k } $ is defined as (\ref{def-zeta-ik}). 
If further assume all clients follow Setup \ref{setup-D-1}, then
\begin{eqnarray} 
 \label{decreasing-property-1}   
 \arraycolsep=1.5pt\def\arraystretch{1.5}
 \begin{array}{lll}
  \L^{k+1} - \L^{k} 
&\leq&   -       \sum_{i=1}^{m} \frac{\mu_{i,k+1 }+\eta-r_i }{2} \E_{\bxi} \|\bx_i^{k+1}-\bx_i^{k}\|^2\\
&  & -\sum_{i=1}^{m} \frac{\eta}{4}  \E_{\bxi}\|\bx^{\tk}-\bx^{\k }\|^2.
    \end{array}  
 \end{eqnarray}   
\end{lemma}

\begin{theorem}\label{global-obj-convergence-inexact}  Under Assumption \ref{ass-fi} and Setup \ref{setup-D-1}, the following results hold. 
 \ \begin{itemize}
 \item[i)]  $\{ \L^k\}$ and $\{\E_{\bxi} F(\bx^{\k},W^{k})\}$  converge. Moreover,
  \begin{eqnarray*}    
 \eqspace{1.75}
 \begin{array}{lll}
\lim\limits_{k\to\infty}\L^{k}= \lim\limits_{k\to\infty} \E_{\bxi} F(\bx^{\k},W^{k}). 
    \end{array}  \end{eqnarray*}  
 \item[ii)] $\lim\limits_{k\to\infty}\E_{\bxi} \|\bx_i^{k+1}-\bx_i^{k}\|^2 = \lim\limits_{k\to\infty} \E _{\bxi}\|\bx^{\tk}-\bx^{\k }\|^2=0$.
% \item[ii)] Any accumulating point is a a stationary point of    (\ref{FL-EPM}). 
 \end{itemize}
 \end{theorem}

 \section{Numerical Experiments}\label{sec:num}
In this section, we conduct some numerical experiments to demonstrate the performance of \FE, which is available at \url{https://github.com/ShenglongZhou/FedEPM}. All numerical experiments are implemented through MATLAB (R2020b) on a laptop with 32GB memory and 2.3Ghz CPU.  

 \subsection{Testing example}
 We consider a classification problem using logistic regression, where local clients have their objective functions as 
\begin{eqnarray*}  
 \arraycolsep=0pt\def\arraystretch{1.5}
~~\begin{array}{llll}
f_i(\bx)= 
\frac{1}{d_i}\sum_{t=1}^{d_i}(\ln (1+{\rm e}^{\langle\ba_i^{t},\bx\rangle} )-b_i^t\langle\ba_i^t,\bx\rangle+\frac{\beta}{2}\|\bx\|^2).
\end{array} 
\end{eqnarray*}
Here, $\ba_i^t\in\R^{n}, b^t_i\in\{0,1\}$ are respectively the $t$th sample and sample label of client $i$ and $\beta>0$ is a penalty parameter (e.g., $\beta=0.001$ in our  numerical experiments).  We use the ``Adult income" dataset \cite{kohavi1996scaling}  from UCI Machine Learning Repository \cite{asuncion2007uci} to demonstrate the performance of \FE. This dataset has  48842 instances and 15 attributes (i.e., 6 continuous and 9 nominal/categorical ones). The processes to generate $(\ba, b)$ are as follows: (i)   
the instances with missing values are removed; (ii) the first 8 categorical attributes are converted into integers; (iii) data are attribute-wisely normalized to have a unit length.  (iv) the last attribute is used to generate labels $b\in\{1,0\}$ by converting $\{> 50k, \leq 50k\}$  into $\{1,0\}$. After these processes, we have $d=45222$ instances and $n=14$. We then randomly divide all instances into $m$ parts with sizes $d_1,\cdots,d_m$, namely, $d:=d_1+\cdots+d_m$.
 
\begin{algorithm} 
\SetAlgoLined
{\noindent  Initialize $k_0, \mu>0$, let $S^0=[m]$,  and set $k = 0$.} 

\For{Every client $i=1,2,\cdots,m $}{
{\noindent  Initializes  $\gamma_i>0$.}
 
  {\noindent Generates a starting point $\bx_i^0$ and  a noisy vector $\bxi_i^0$.}
 
 {\noindent Uploads $\bz_i^{0}=\bx_i^0+\bxi_i^0$ to the server.} 
}

\For{$k=0,1,2,\cdots $}{

\If{$  k\in\K:=\{0,k_0,2k_0,3k_0,\cdots \}$}{

{The  server randomly selects $\S^{\tk}\subseteq [m]$ and  broadcasts $\bx^{\tk} $ to clients in $\S^{\tk}$, where }  
\begin{eqnarray}\label{FP-sub1}
\eqspace{1.5} 
 \begin{array}{lcll}
\bx^{\tk} =     \frac{1}{|\S^{\tk}|}\sum_{i\in \S^{\tk}}  \bz^{\k}_i.
\end{array}
\end{eqnarray}
 }

\For{every $i\in \tauk$}{
 {\noindent Client $i$ updates its parameter by}
  
 [\SFA]
\begin{eqnarray} 
\label{fedvag}   
\eqspace{1.25}
\bx^{k+1}_i {=} \left\{ \begin{array}{llll}
\bx^{\tk} {-}  \gamma_i   \nabla  f_i(\bx^{\tk}),&~ k{\in}\K, \\
\bx_i^k {-}   \gamma_i     \nabla f_i(\bx_i^k),&~ k{\notin}\K. 
    \end{array}\right.
    \end{eqnarray}
    
 [\SFP]  
    \begin{eqnarray} 
\label{fedprox} 
\eqspace{1.25}
\begin{array}{lcll}
\bx^{k+1}_i {\approx} {\rm argmin} f_i(\bx_i){+}\frac{\mu}{2}\|\bx_i-\bx^{\tk}\|^2.
\end{array}
\end{eqnarray}
\If{$  k+1\in\K$}{
{\noindent Generates a randomly noisy vector $ \bxi_i^{\tk}$.}

{\noindent Sends $\bz^{\tk}_i$ to the server, where}
 \begin{align}
\label{iceadmm-sub4-1}  
\bz^{\tk}_i&=\bx_i^{k+1}  + \bxi_i^{\tk}. 
\end{align}  
}
 }
 
 \For{every $i\notin \tauk$}{
 {\noindent Client $i$ keeps $\bx^{k+1}_i=\bx^{k}_i$.} 
 }
 
}
\caption{\SFA\ and \SFP. \label{algorithm-FA}}
\end{algorithm}
\subsection{Benchmark algorithms and implementations}
We will compare our proposed method with \SFA\ \cite{li2020secure} and the modified version (\SFP) of \FP\ \cite{li2020federatedprox}. Their frameworks are given in Algorithm \ref{algorithm-FA}.
\begin{itemize}[leftmargin=*] 
\item For \SFA,  we make use of the full gradient rather than the stochastic gradient to ensure fair comparison. To do that, we only need to set the mini-batch size for each client $i$  as $d_i$. In addition, instead of fixing $\gamma_i$, we update it by
 \begin{eqnarray} \label{gamma-i-k}
 \eqspace{1.5}
\begin{array}{lll}\gamma_i:=\gamma_i^{k}:=\frac{2d_i}{\sqrt{2k_0+\lfloor k/k_0\rfloor}}.\end{array} 
\end{eqnarray}
\item  For \FP,  to  maintain the privacy, we also add noises to the parameters when they are uploaded to the server. So we call it the secured  \FP, dubbed as \SFP.  As required,  we need to approximately solve \eqref{fedprox}. To proceed with that, we employ Algorithm \ref{alg4-sub} in which $\mu=10^{-5}$ and $\ell$ is set as a small integer (e.g. $\ell=3$) for accelerating the computational speed in the numerical experiments.
 \begin{algorithm} [!th]
\SetAlgoLined
{\noindent  Input $ \bx^{\tk}, \bx_i^k, \ell,$ and $\gamma_i$ as \eqref{gamma-i-k}. Let $ {\bv}_i^1$ be given by}
\begin{eqnarray*}  
\eqspace{1.25}
 {\bv}_i^1 =\left\{ \begin{array}{llll}
\bx^{\tk}  ,&~ {\rm if} ~k \in \K, \\
\bx_i^k ,&~ {\rm if} ~ k \notin \K. 
    \end{array}\right.
    \end{eqnarray*} 
\For{$t=1,2,\cdots,\ell$}{
%{\noindent  Update   $\bv^{t+1}_i$ by}
\begin{eqnarray*} 
 \eqspace{1.25}
\begin{array}{lll}\bv^{t+1}_i %&=& {\rm argmin}~ f_i( {\bv}_i^t)+\langle \nabla f_i( {\bv}_i^t), \bv_i- {\bv}_i^t\rangle \\
%&+& \frac{1/\gamma_i-\mu}{2}\|\bv_i- {\bv}_i^t\|^2 + \frac{\mu}{2}\|\bv_i-\bx^{\tk}\|^2\\
&=& {\bv}_i^t -\gamma_i(\nabla f_i( {\bv}_i^t) + \mu({\bv}_i^t-\bx^{\tk})).\end{array} 
\end{eqnarray*}
}
{\noindent  Output $\bx^{k+1}_i=\bv^{\ell+1}_i$}
\caption{Solving problem \eqref{fedprox} inexactly.\label{alg4-sub}}

\end{algorithm}  
\item  For \FE, there are two constants in model \eqref{FL-EPM} that should be finely tuned for different examples for achieving better performance, which can be done by the well known  cross-validation. However, we fix them as $\lambda=\eta/2$ and $\eta=(0.02m+1)(\rho+0.1)10^{-5}$  in the sequel for simplicity.
\end{itemize}
All algorithms start with the same initial point,  $\bx_i^0=0$, and terminate if solution
 $\bx^{\tau_{k}}$ satisfies  $\|\nabla f(\bx^{\k})\|^2 < 10^{-6}$  or
 \begin{eqnarray*} 
 \arraycolsep=1.0pt\def\arraystretch{1.5}
\begin{array}{lll}
 {\rm var}\{ f(\bx^{\k-3}), f(\bx^{\k-2}),f(\bx^{\k-1}),f(\bx^{\k})\}  \leq   \frac{n 10^{-8}}{1+|f(\bx^{\k})|},\\
\end{array} 
\end{eqnarray*}
where ${\rm var}(\bx)$ calculates the variance of $\bx$. Moreover, we use the same way to generate $\S^{\tk}$.   Specifically, $S^1, \S^2,\cdots$ are selected independently with $|\S^\tau|=\rho m$ for any $\tau\geq1$, where $\rho\in(0,1]$. Indices in each $\S^\tau$ are uniformly sampled from $[m]$ without replacement. Finally, noise vectors $\bxi_i^{\tk}$ are generated as, for any $j\in[n]$, 
\begin{eqnarray} \label{noise-generation}
  \arraycolsep=1.4pt\def\arraystretch{1.75}
\begin{array}{lll}
\epsilon_{ij}^{\tk}&\sim& {\rm Lap} \Big(0,\frac{  2\|\bg_i^{\tk}\|_1}{  \varepsilon \mu_{i,0 } (1+c_i\|\bx_i^k-\bx^{\tk}\|^2) \alpha_i^{k+1} }\Big),\\%=  {\rm Lap}\Big(0,\frac{ 100\overline{\Delta}_i^{k+1}}{ 1.001 ^{k+1} }\Big),\\
% {\delta}_i^{k+1} &:=& 2\med(|  {g}_{i1}^{\tk}|, \cdots,| {g}_{in}^{\tk}|),
\end{array}\end{eqnarray} 
where  %${g}_{ij}^{\tk}$ is the $j$th entry of $\bg_i^{\tk}$, 
$\mu_{i,0}= 0.05, c_i=10^{-8}$, and $\alpha_i=1.001$. Here,  we calculate $ 2\|\bg_i^{\tk}\|_1$  to bound $\Delta_i^{\tk}$ since the latter is not easy to compute. In the following numerical simulation, we alter
\begin{itemize}[leftmargin=*] 
\item  $k_0\in\{4,8,12,16,20\}$ to see the communication efficiency;
 \item  $\rho\in(0,1)$ to see the effect of partial devices participation;
  \item  $\varepsilon\in(0,1)$ to see the effect of privacy preserving.
\end{itemize}
%\begin{eqnarray}\label{scale-para}
%  \arraycolsep=0pt\def\arraystretch{1.5}
%\begin{array}{lll}
% \Delta_i^{k+1}&\overset{\eqref{scale-para}}{=}&   \max\limits_{\D'_i \in \widehat{\D}_i } \|\blg^{k+1}_i(\D'_i)- \blg^{k+1}_i(\D_i)\|_1,\\
% &\leq&   \max\limits_{\D'_i \in \widehat{\D}_i } \|\blg^{k+1}_i(\D'_i)- \blg^{k+1}_i(\D_i)\|_1,\\
%\end{array}\end{eqnarray}  
%More specifically, let entries of $\bxi_i^{\tk}$ be i.i.d. random variables from
%\begin{eqnarray*} 
%  \arraycolsep=1.4pt\def\arraystretch{1.5}
%\begin{array}{lll}
%\epsilon_{ij}^{\tk}\sim {\rm Lap}\Big(0, \frac{1}{\varepsilon \mu_{i,k+1 } } \Big), ~~j\in[n],
%\end{array}\end{eqnarray*}
%where $ \mu_{i,0}=10$

\subsection{Numerical comparison}
 We will report the following five factors to demonstrate  the performance of the three algorithms, 
\begin{eqnarray*}
 \eqspace{1.5} 
 \begin{array}{lll}
(f(\bx)/m,~\text{CR,~ TCT, ~LCT,~ SNR}),
\end{array}\end{eqnarray*} 
where $\bx$ is the final obtained solution, and the last four factors respectively represent the communication rounds, the total computational time (in second), the local computational time (in second) by clients between two consecutive communication rounds, and the signal-to-noise-ratio defined as
\begin{eqnarray*}
 \eqspace{1.5} 
 \begin{array}{lll}
{\rm SNR}:=\min_{i\in[m]}\log_{10} (\|\bx_i^{k+1}\|/\|\bxi_i^{\tk}\|).
\end{array}\end{eqnarray*} 
Here, $\bx_i^{k+1}$ and $\bxi_i^{\tk}$ are the point and noise produced by each algorithm in the last iteration. It is worth mentioning that LCT is useful to illustrate the computation efficiency for local clients, namely, the shorter LCT the lower computational complexity. In addition, the smaller {\rm SNR} indicates the higher privacy to be maintained.
\subsubsection{Accuracy} From Fig. \ref{fig:cr-obj}, we can observe that the three algorithms eventually approach the same objective function values when  $(\varepsilon,\rho)=(0.1,0.5), m\in\{50,100\},$ and $k_0\in\{4,12,20\}$. However, \FE\ declines the fastest and hence uses the fewest CR for each $k_0$. In the sequel, we will not report $f(\bx)/m$ obtained by three algorithms as they are basically similar to each other.

\begin{figure}[!th]
	 \begin{subfigure}{.24\textwidth}
	\centering
	\includegraphics[width=1.02\linewidth]{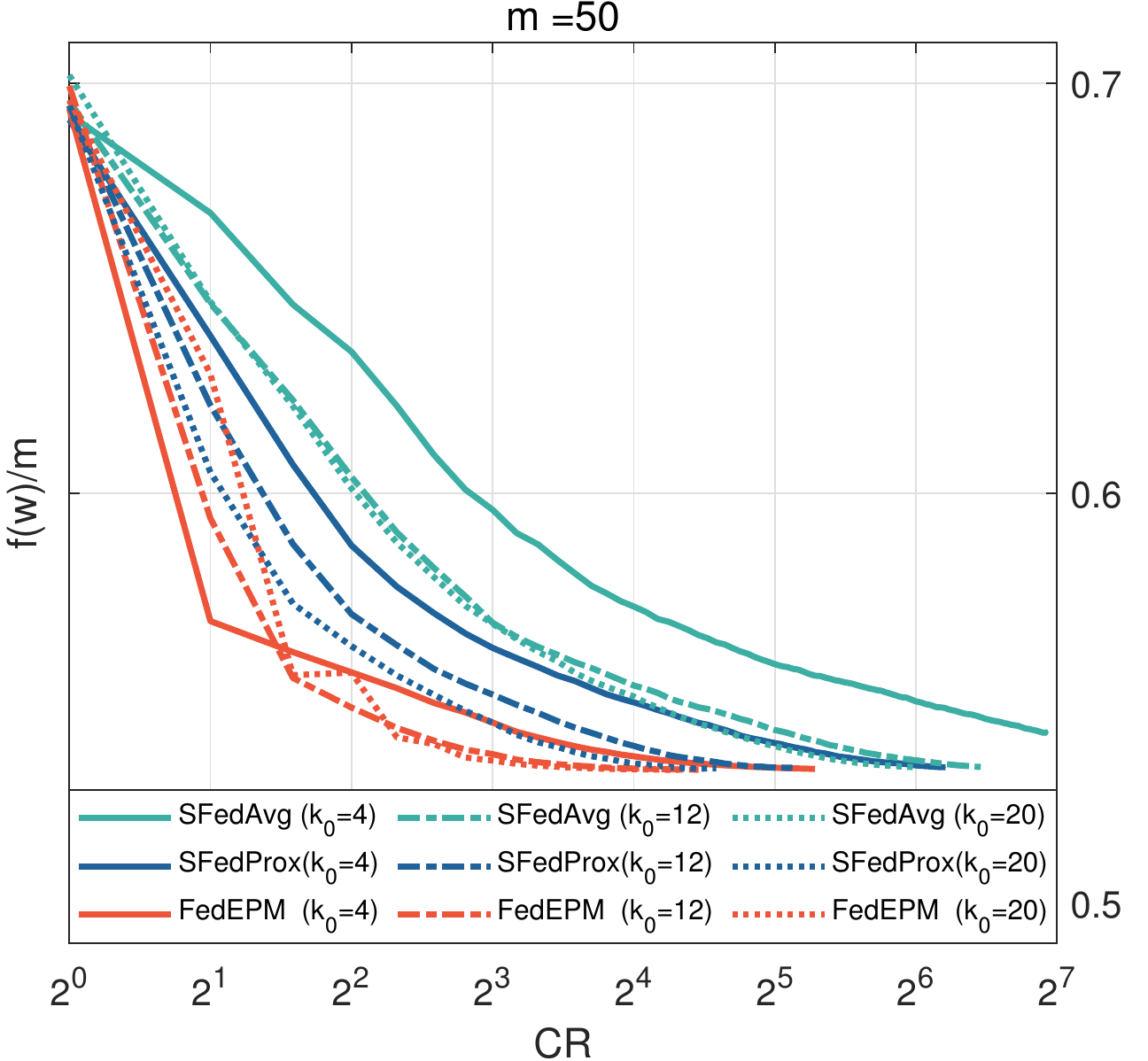}
	%\caption{{\tt CEADMM} solving Example \ref{ex-linear}}
	%\label{fig:CEADMM-diff}
\end{subfigure}	
\begin{subfigure}{.24\textwidth}
	\centering
	\includegraphics[width=1.02\linewidth]{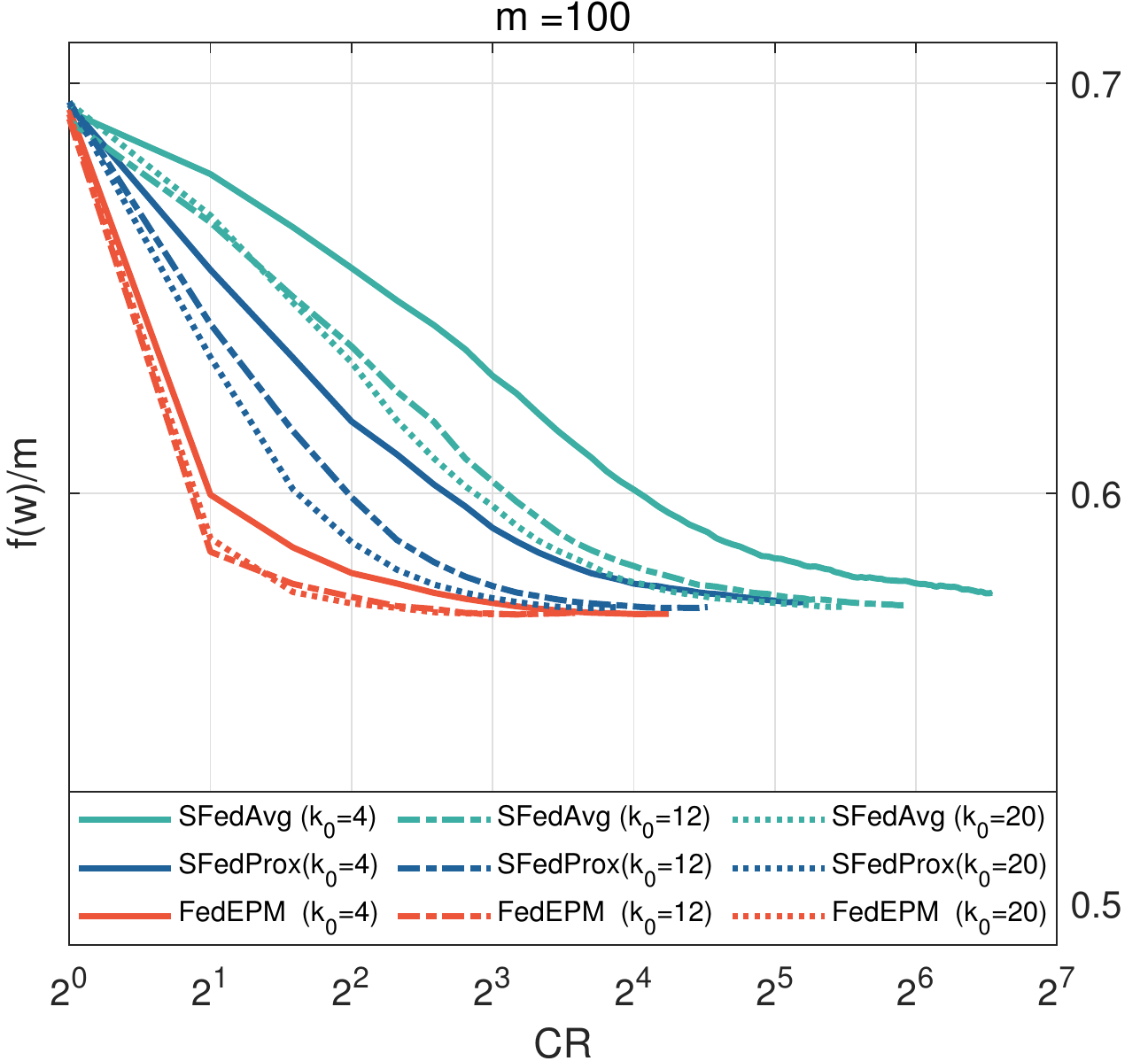}
	%\caption{{\tt ICEADMM}  solving Example \ref{ex-linear}}
	%\label{fig:ICEADMM-diff}
\end{subfigure}   
\caption{TCT v.s. $k_0$.\label{fig:cr-obj}}
\end{figure}

\subsubsection{Communication and computation efficiency} We first  fix $(\varepsilon,\rho)=(0.1,0.5)$ and choose $m\in\{50,100\},$ $k_0\in\{4,8,12,16,20\}$. For each fixed $(\varepsilon,\rho,m,k_0)$, we run 100 trails and record the average results.  It can be clearly seen from Fig. \ref{fig:k0-tct} that the bigger $k_0$, the fewer CR and the shorter TCT. Apparently, \FE\ uses the fewest CR and runs the fastest, displaying the highest communication  efficiency.  In addition,  the data in TABLE \ref{tab:k0-lct} illustrate that \FE\ consumes much lower LCT, which indicates it has the highest computation efficiency among the three algorithms. Moreover, although \SFP\ consumes fewer CR than \SFA\ (see Fig. \ref{fig:k0-tct}), it is not computation-efficient based on the LCT data in the table. This is because every selected client has to  solve sub-problems \eqref{fedprox} by Algorithm \ref{alg4-sub} at each iteration.

\begin{figure}[!th]
	 \begin{subfigure}{.24\textwidth}
	\centering
	\includegraphics[width=1.01\linewidth]{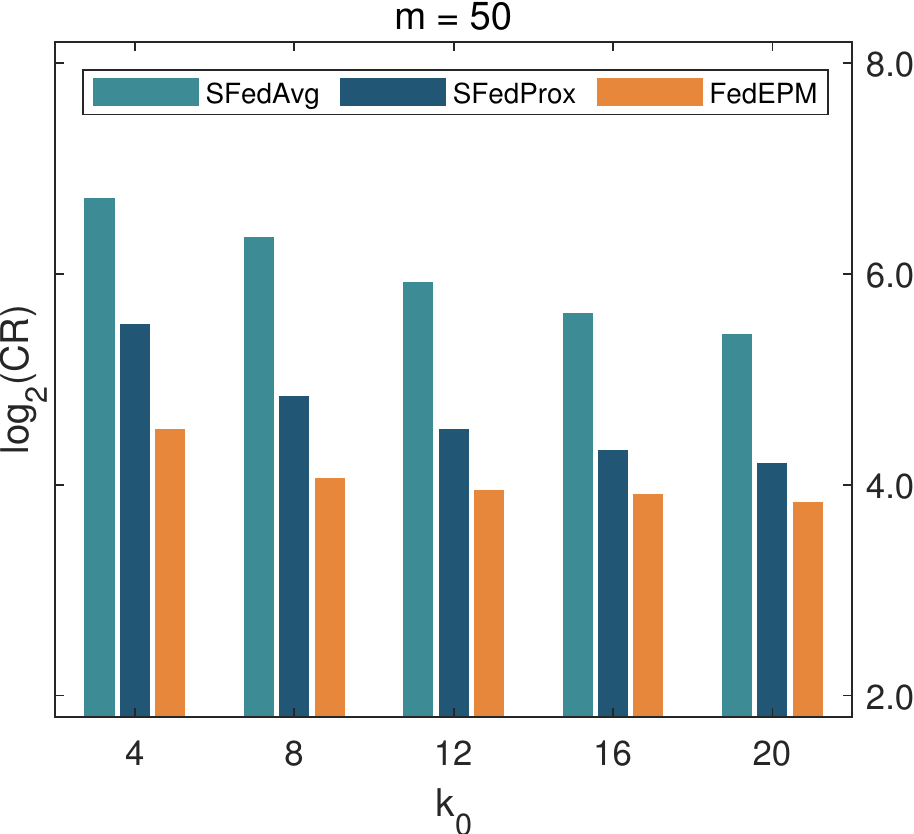}
	%\caption{{\tt CEADMM} solving Example \ref{ex-linear}}
	%\label{fig:CEADMM-diff}
\end{subfigure}	 
\begin{subfigure}{.24\textwidth}
	\centering
	\includegraphics[width=1.01\linewidth]{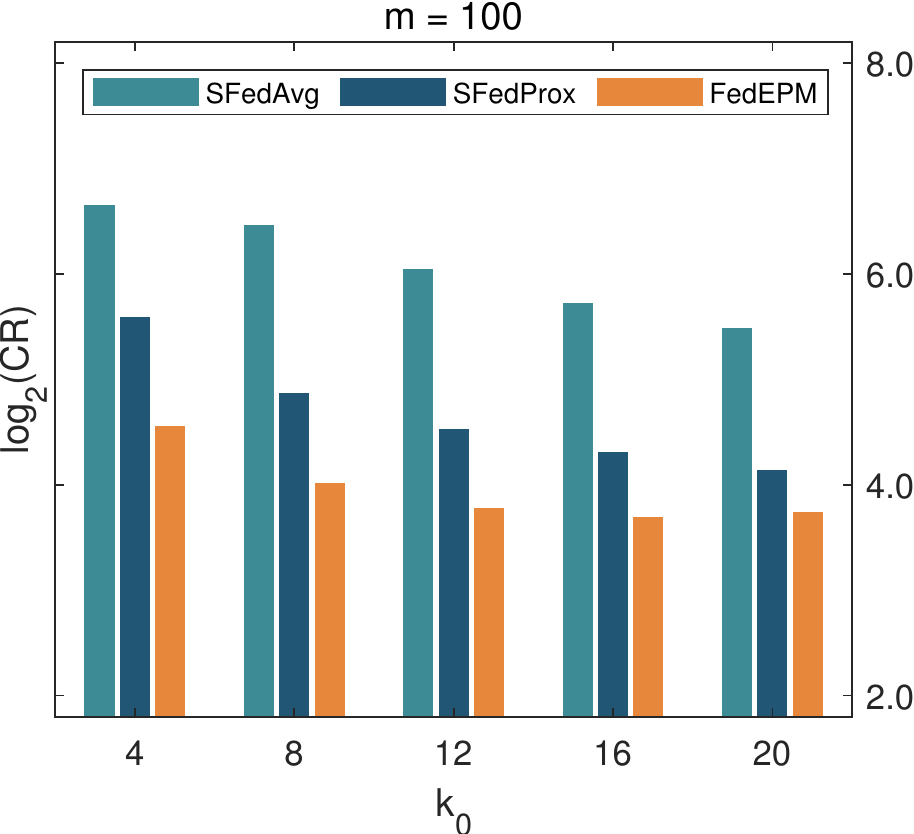}
	%\caption{{\tt ICEADMM}  solving Example \ref{ex-linear}}
	%\label{fig:ICEADMM-diff}
\end{subfigure}   \\\vspace{2mm}

	 \begin{subfigure}{.24\textwidth}
	\centering
	\includegraphics[width=1.01\linewidth]{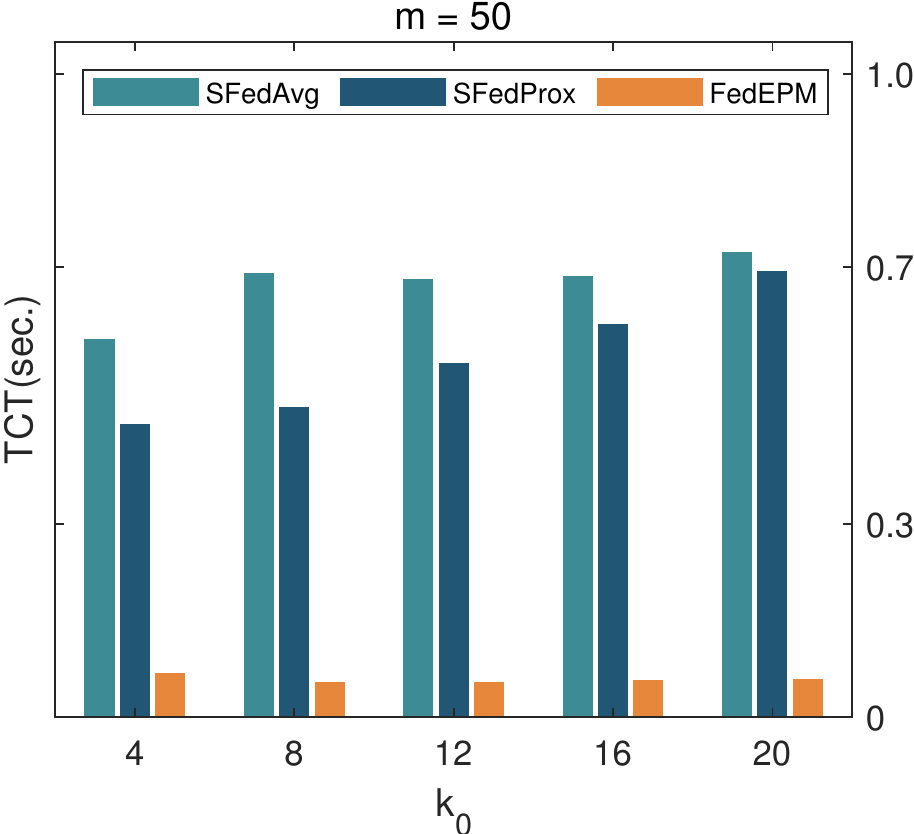}
	%\caption{{\tt CEADMM} solving Example \ref{ex-linear}}
	%\label{fig:CEADMM-diff}
\end{subfigure}	
\begin{subfigure}{.24\textwidth}
	\centering
	\includegraphics[width=1.01\linewidth]{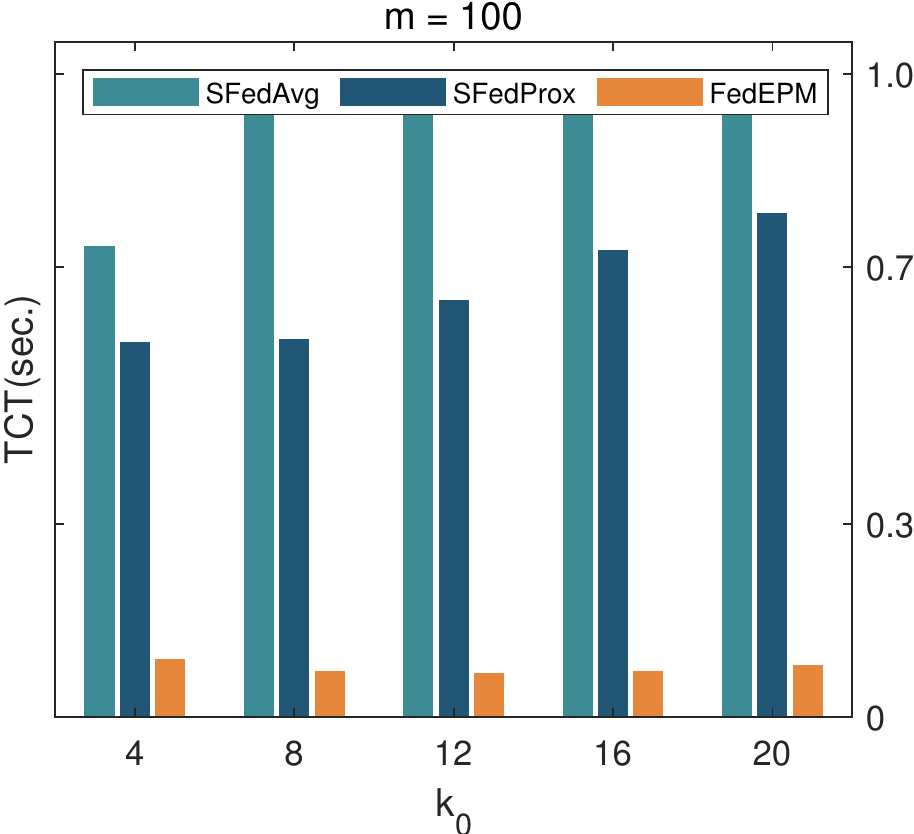}
	%\caption{{\tt ICEADMM}  solving Example \ref{ex-linear}}
	%\label{fig:ICEADMM-diff}
\end{subfigure}   
\caption{Effect of $k_0$.\label{fig:k0-tct}}
\end{figure}

\begin{table}[!th]
	\renewcommand{\arraystretch}{1.25}\addtolength{\tabcolsep}{-4pt}
	\caption{LCT v.s. $k_0$.}\vspace{-3mm}
	\label{tab:k0-lct}
	\begin{center}
		\begin{tabular}{ccccccccc }
			\hline
	&	\multicolumn{3}{c}{$m=50$}					&&	\multicolumn{3}{c}{$m=128$}					\\\cline{2-4}\cline{6-8}
$k_0$	&	\SFA	&	\SFP	&	\FE	&&	\SFA	&	\SFP	&	\FE	\\\cline{1-4}\cline{6-8}
4	&	0.0037 	&	0.0082 	&	0.0011 	&&	0.0058 	&	0.0109 	&	0.0019 	\\
8	&	0.0067 	&	0.0155 	&	0.0015 	&&	0.0097 	&	0.0198 	&	0.0024 	\\
12	&	0.0096 	&	0.0226 	&	0.0018 	&&	0.0132 	&	0.0279 	&	0.0030 	\\
16	&	0.0125 	&	0.0293 	&	0.0021 	&&	0.0172 	&	0.0359 	&	0.0036 	\\
20	&	0.0148 	&	0.0360 	&	0.0023 	&&	0.0213 	&	0.0457 	&	0.0041 	\\\hline
 		\end{tabular}
	\end{center}
	\vspace{-5mm}
\end{table}

\begin{figure*}[!th] 
	\centering
	\includegraphics[width=1.0\linewidth]{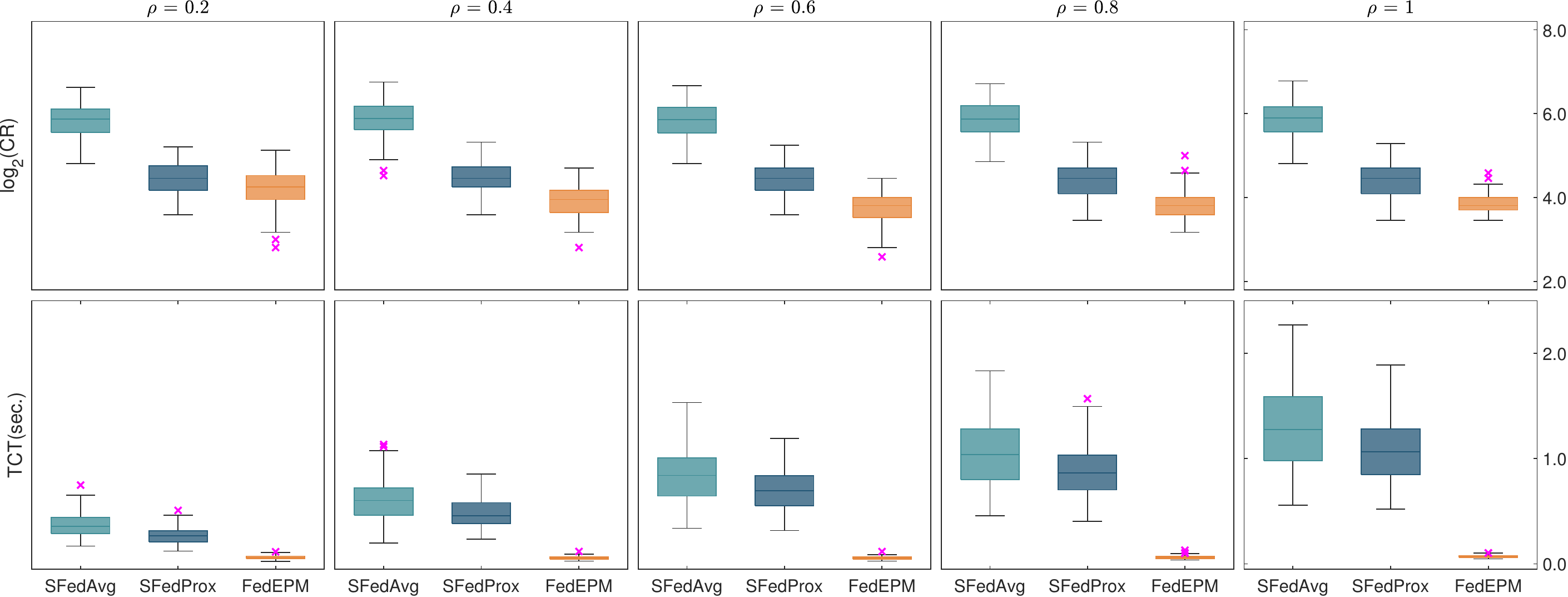} 
\caption{Effect of  $\rho$.\label{fig:rho-tct}}
\end{figure*}

 \begin{figure*}[!th] 
	\centering
	\includegraphics[width=1.\linewidth]{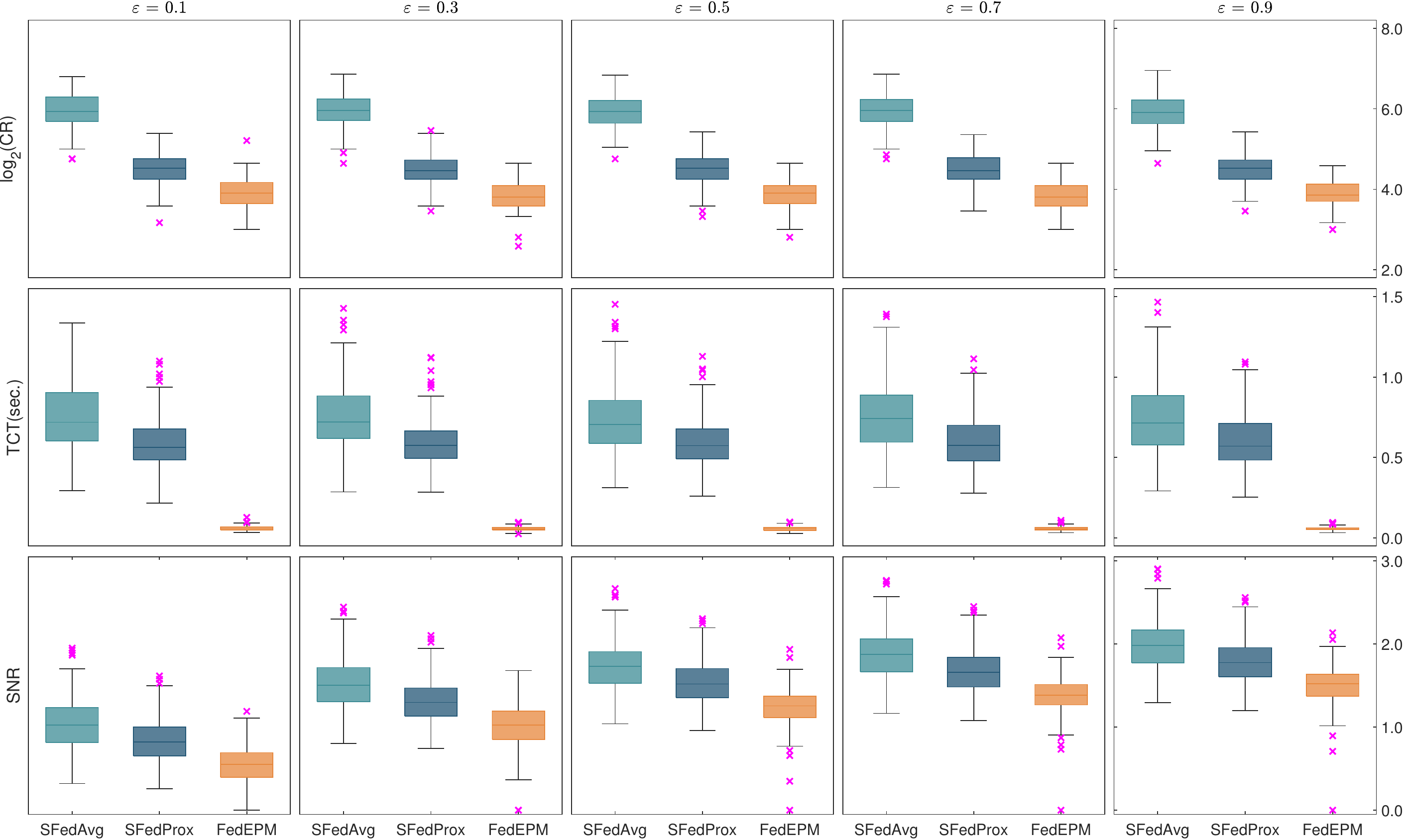} 
\caption{Effect of $\varepsilon$.\label{fig:var-tct}}
\end{figure*}

\subsubsection{Partial devices participation}To see the robustness of three algorithms to the number of selected devices per each round of communication, we fix $(\varepsilon,k_0,m)=(0.1,12, 50)$ but alter $\rho\in\{0.2,0.4,0.6,0.8,1.0\}$. For each fixed $(\varepsilon,\rho,m,k_0)$, we run 100 trails and draw box plots in Fig. \ref{fig:rho-tct}. The central line in each box indicates the median while the bottom and top edges of each box present the $25$th and $75$th percentiles, respectively. The outliers are plotted individually using the `x' symbols.  As expected, CR is slightly decreasing and TCT is increasing along with the rising of  $\rho$. It is evident that  \FE\ consumes the lowest median of CR and TCT, delivering the highest communication and computation efficiency. %(ii) \FE\ is more steady than the other two algorithms since it does not generate outliers. For the other two algorithms, the smaller $\rho$ is, the more outliers are produced. (iii)  All algorithms take the increasing computational time when $\rho$ rises, as expected. 

\subsubsection{Privacy preserving} For the aim of maintaining $\varepsilon$-differential privacy, we generate the noise as \eqref{noise-generation} where $\varepsilon$ has an impact on the scale of the added noise. The smaller $\varepsilon$ the larger noise, which leads to stronger privacy. To see this, we fix $(\rho,k_0,m)=(0.5,12,50)$ but alter $\varepsilon\in\{0.1,0.3,0.5,0.7,0.9\}$. For each fixed $(\varepsilon,\rho,m,k_0)$, again we run 100 trails and draw box plots in Fig. \ref{fig:var-tct}. One can observe that the change of $\varepsilon$ does not impact CR and TCT greatly. However, it can be clearly seen that SNR is increasing with $\rho$ ascending, thereby weakening privacy.  When making comparison among the three algorithms, \FE\ outperforms the other two once again. In particularly, it can maintain higher privacy than the others since it achieves the smallest SNR.

\section{Conclusion}
We leveraged the exact penalty method to reformulate the FL problem in a centralized form, which was effectively solved by the alternating direction method. The framework of the proposed algorithm was flexible to integrate tactics to improve communication and computation efficiency, eliminate the stragglers' effect, and preserve clients' privacy. These were testified by numerical comparisons with two state-of-the-art algorithms. Thanks to its abilities, it might be relatively practical for many applications. In addition,  we feel that the exact penalty may also be beneficial for processing de-centralized FL problems.  We leave these as the future research.

%\section*{Acknowledgments}
%This should be a simple paragraph before the References to thank those individuals and institutions who have supported your work on this article.

\appendix
%[Proof of all Theorems]
%Use $\backslash${\tt{appendix}} if you have a single appendix:
%Do not use $\backslash${\tt{section}} anymore after $\backslash${\tt{appendix}}, only $\backslash${\tt{section*}}.
%If you have multiple appendixes use $\backslash${\tt{appendices}} then use $\backslash${\tt{section}} to start each appendix.
%You must declare a $\backslash${\tt{section}} before using any $\backslash${\tt{subsection}} or using $\backslash${\tt{label}} ($\backslash${\tt{appendices}} by itself
  
 \section*{Proof  of Lemma \ref{lemma-solution-EN}}
\begin{proof} Problem \eqref{x-s-*} is equivalent to 
\begin{eqnarray}\label{x-s-*-down}
 \eqspace{1.5} 
 \begin{array}{rcll}
w(s^*) &=& {\rm argmin}_w \sum_{i=1}^m (\lambda|w-w_i^{\downarrow}| + \frac{\eta}{2}(w-w_i^{\downarrow})^2).
\end{array}\end{eqnarray} 
A point $w^*$ is the optimal solution  to \eqref{x-s-*-down} if and only is it satisfies the following optimality condition,
\begin{eqnarray}\label{x-s-*-opt}
 \eqspace{1.5} 
 \begin{array}{lll}
0&=&  \sum_{i=1}^m (\lambda \pi_i^* +  \eta (w^*-w_i^{\downarrow})),\\
 \pi_i^* &\in& \sgn(w^*-w_i^{\downarrow}).
\end{array}\end{eqnarray} 
If $w^*>w^\downarrow_1$, then  $w^*{-}w_i^{\downarrow}> 0$ and $ \pi_i^*=1$ for all $i\in[m] $, thereby $\lambda \pi_i^* +  \eta (w^*-w_i^{\downarrow})>0$ for all $i\in[m] $, which contradicts with the first condition in \eqref{x-s-*-opt}. Similarly, we conclude that $w^*<w^\downarrow_m$ does not hold.  Overall, $w^*\in[w^\downarrow_m,w^\downarrow_1]$. That is, one of the following two cases is valid.
\begin{itemize}
\item[c1)] there is an $s^*\in[1,m]$ such that \begin{eqnarray}\label{x-s-*-opt-w}
 \eqspace{1.5} 
 \begin{array}{lll}
w_{1}^{\downarrow} \geq \cdots \geq w_{s^*}^{\downarrow}> w^*
> w_{s^*+1}^{\downarrow}\geq \cdots \geq w_{m}^{\downarrow}.
\end{array}\end{eqnarray}
\item[c2)]  there is an $s^*\in[1,m]$ such that $w^*=w_{s^*}^{\downarrow}$.
\end{itemize}
If case c1) is true,  then  
\begin{eqnarray} \label{pi-s-1-1}
\pi_i^* =\left\{  \begin{array}{rl}
-1,~~&   i=1,2,\cdots,s^*,\\ 
1,~~&   i=s^*+1,\cdots,m,
\end{array}
\right.
\end{eqnarray} 
 due to \eqref{x-s-*-opt-w} and 
\begin{eqnarray*} 
 \eqspace{1.5} 
 \begin{array}{rl} \sum_{i=1}^m\pi_i^*= -s^*+(m-s^*)=m-2s^*.
\end{array}
\end{eqnarray*} 
Using this fact, we can obtain  the desired result due to
\begin{eqnarray*}
 \eqspace{1.75} 
 \begin{array}{rcll}
 w^*&\overset{ \eqref{x-s-*-opt}}{=}&  
 \frac{1}{m} \sum_{i=1}^m  w_i^{\downarrow}-\frac{\lambda}{\eta m}\sum_{i=1}^m    \pi_i^*\\
&=&  \frac{1}{m}\sum_{i=1}^m w_i  + \frac{\lambda (2s^*-m)}{ \eta m }.
\end{array}\end{eqnarray*}
If case c2) is true, we can find the optimal solution by 
\begin{eqnarray*}
 \eqspace{1.5} 
 \begin{array}{rcll}
w^*=\underset{w\in\{w_1^{\downarrow},\cdots,w_m^{\downarrow}\}}{\rm argmin}\sum_{i=1}^m (\lambda|w-w_i^{\downarrow}| + \frac{\eta}{2}(w-w_i^{\downarrow})^2).
\end{array}\end{eqnarray*} 
The whole proof is completed.
\end{proof}

\section*{Proof  of Theorem \ref{the-privacy}}
\subsection{A useful lemma}
 To derive Theorem \ref{the-privacy}, we need the following lemma.
  \begin{lemma} For any given $a>0$, it hold
\begin{eqnarray}  \label{Lip-continuity-sfot}
 \eqspace{1.5}  \begin{array}{llll}
|\soft(t,a)-\soft(t',a)|\leq 2|t-t'|.
 \end{array}
\end{eqnarray} 
\end{lemma}
\begin{proof} It follows from \eqref{def-soft-1d} that \begin{eqnarray} \label{def-soft-1d-0}
 \eqspace{1.25}
\soft(t,a) =\left\{ \begin{array}{lcrr}
 t-a, ~~& t&>&a,\\  
 0, & |t|&\leq&a,\\  
 t+a, & t&<&-a. 
\end{array}\right.
\end{eqnarray}
There are 9 cases of $\delta:=|\soft(t,a)-\soft(t',a)|$ which are summarized in Table \ref{nine-cases}.
\begin{table}[!th]
	\renewcommand{\arraystretch}{1.25}\addtolength{\tabcolsep}{10pt}
	\caption{Nine cases of $|\soft(t,a)-\soft(t',a)|$.}\vspace{-3mm}
	\label{nine-cases}
	\begin{center}
		\begin{tabular}{cccc}
			\hline
 &$t'>a$&	 $|t'|\leq a$	&	$t'<-a$\\\hline
$t>a$&	C1	&	C2	 	&	C3 	\\
$|t|\leq a$& C7	&	C4	&	C5 	\\
$t<-a$ &	 C8	&	 C9	&	C6 	\\
\hline
 		\end{tabular}
	\end{center}
	\vspace{-5mm}
\end{table}
To prove the conclusion, we only show cases C1-C6 since cases C7, C8, C9 are similar to cases C2, C3, C5, respective.
 
 \begin{itemize}[leftmargin=*]
\item For cases C1 \& C6, by \eqref{def-soft-1d-0}, it has $\delta=|t-t'|$.
\item For case  C2, by \eqref{def-soft-1d-0}, it has 
\begin{eqnarray*} 
\delta =   t-a \leq t-|t'|=|t|-|t'| \leq |t-t'|.
\end{eqnarray*}
\item For case  C3, by \eqref{def-soft-1d-0}, it has 
\begin{eqnarray*} 
\delta =   |t-a -t'-a |\leq |t-t'| + 2a \leq 2 |t-t'|.
\end{eqnarray*}
\item For case  C4, by \eqref{def-soft-1d-0}, it has $\delta = 0 \leq  |t-t'|$.
\item For case  C5, by \eqref{def-soft-1d-0}, it has 
\begin{eqnarray*} 
\delta =   |t'+a |= -t'-a \leq -t'-|t|  \leq   |t-t'|.
\end{eqnarray*}
\end{itemize}
Overall, we show $\delta\leq 2|t-t'|$.
\end{proof}
  
  \subsection{Proof of Theorem \ref{the-privacy}}
 \begin{proof} We only need to focus on $k+1\in\K$ as there is no  noise generated when $k+1\notin\K$. First, by \eqref{iceadmm-sub2},  we denote 
 \begin{eqnarray}\label{tilde-x-D}
  \eqspace{1.5}
\begin{array}{lll}
\widetilde{\bx}_i^{k+1}(\D_i) &:=& \mu_{i,k+1} (\bx_i^k - \bx^{\tk})- \bg_i^{\tk}(\D_i),\\
   \bx^{k+1}_i(\D_i)  &:= &\bx^{\tk}-\soft (\widetilde{\bx}_i^{k+1}(\D_i), \lambda  )/(\eta+\mu_{i,k+1 }).
\end{array}\end{eqnarray}
 It follows from \eqref{iceadmm-sub4} that $\bz_i^{\tk}= \bx_i^{k+1}(\D_i) +  \bxi_i^{\tk}$ and thus
 \begin{eqnarray*}
 \arraycolsep=1.4pt\def\arraystretch{1.5}
  \begin{array}{lcl} 
(\bxi_i^{\tk})' &:=& \bz_i^{\tk}-\bx_i^{k+1}(\D'_i),\\       
&=&\bz_i^{\tk} - \bx_i^{k+1}(\D_i)+ \bx_i^{k+1}(\D_i) -\bx_i^{k+1}(\D'_i)    \\
&=&  \bxi_i^{\tk}+\bx_i^{k+1}(\D_i) -\bx_i^{k+1}(\D'_i),
   \end{array}
  \end{eqnarray*} 
  which results in  
     \begin{eqnarray}\label{gap-xi-k}
 \arraycolsep=1.4pt\def\arraystretch{1.5}
  \begin{array}{lcl}
&&  \|\bxi_i^{\tk}-(\bxi_i^{\tk})'\|_1 \\
&=& \|\bx_i^{k+1}(\D_i)-\bx_i^{k+1}(\D'_i)\|_1\\
&\overset{\eqref{tilde-x-D}}{=}& \frac{1 }{ \eta+\mu_{i,k+1 }}\| \soft (\widetilde{\bx}_i^{k+1}(\D_i), \lambda  ) - \soft (\widetilde{\bx}_i^{k+1}(\D'_i), \lambda  ) \|_1  \\
&\overset{\eqref{Lip-continuity-sfot}}{\leq}&  
\frac{2}{\eta+\mu_{i,k+1 }}\|\bg_i^{\tk}(\D_i)-\bg_i^{\tk}(\D'_i)\|_1\\
&\overset{\eqref{scale-para}}{\leq}&   \frac{2\Delta_i^{k+1}}{ \eta+\mu_{i,k+1 }} \leq  \frac{2\Delta_i^{k+1}}{ \mu_{i,k+1 }} =:\varrho .
\end{array}\end{eqnarray}
 We note from  \eqref{sample-noise} that  
     \begin{eqnarray}\label{decompose-P-x}
 \arraycolsep=1.4pt\def\arraystretch{1.75}
  \begin{array}{lcl} 
{\rm ln}\frac{\P(\bz^{\tk}|\D_i) }{\P(\bz^{\tk}|\D'_i) } &=& {\rm ln} \frac{\P(\bxi_i^{\tk}) }{\P((\bxi_i^{\tk})')} = \sum_{j=1}^n {\rm ln} \frac{\P(\epsilon_{ij}^{\tk}) }{\P((\epsilon_{ij}^{\tk})')} \\
&\overset{\eqref{sample-noise}}{=}&  \sum_{j=1}^n \frac{ |(\epsilon_{ij}^{\tk})'| -|\epsilon_{ij}^{k+1}|}{\varrho/\varepsilon }\\
&\leq&\frac{ \|(\bxi_i^{\tk})'-\bxi_i^{k+1}\|_1}{\varrho/\varepsilon} \\
& \overset{\eqref{gap-xi-k}}{\leq}&  \varepsilon, 
   \end{array}
  \end{eqnarray} 
 showing the $\varepsilon$-differential privacy.
 \end{proof}
\section*{Proofs of theorems in Section \ref{sec:ADMMFL}}
\subsection{Some Useful Properties}
Suppose $f$ is Lipschitz continuous with $r>0$. Then for any $\bx_1, \bx_2  $, and $\bx(t)  := \bx_2  +t(\bx_1-\bx_2  ) $, where $t\in(0,1)$,  the Mean Value Theorem  suffices to, for any $\bx_i\in\{\bx_1,\bx_2\}$
\begin{eqnarray}  \label{H-Lip-continuity-fxy}
 \eqspace{1.5}  \begin{array}{llll}
&&f(\bx_1) - f(\bx_2    ) -\langle \nabla  f(\bx_i  ), \bx_1-\bx_2     \rangle \\
&=&  \int_0^1 \langle \nabla  f(\bx(t)) - \nabla  f(\bx_i ), \bx_1-\bx_2     \rangle dt  \\
&\leq&  \int_0^1 r\|\bx(t) - \bx_i\|\| \bx_1-\bx_2     \| dt\\
%&=&r\| \bx_1-\bx_2\|^2\Bigg\{  
%\begin{array}{lll}
%\int_0^1 (1-t)dt,&~~{\rm if}~&\bx_i=\bx_1\\
%\int_0^1 tdt,&~~{\rm if}~&\bx_i=\bx_2\
%\end{array} \\
&=& \frac{r}{2}\| \bx_1-\bx_2     \|^2.
 \end{array}
\end{eqnarray}  
For any $t>0$ and vectors $\bx,\bz$, we have
\begin{eqnarray} \label{two-vecs}
 \arraycolsep=1.5pt\def\arraystretch{1.5}
 \begin{array}{lcl} 
 2\langle \bx , \bz\rangle &\leq &  t\|\bx\|^2 + (1/t)\|\bz\|^2.	
% 2\langle \bx_1, H\bx_2\rangle &\leq &  t\|\bx_1\|^2_H + (1/t)\|\bx_2\|^2_H, \\
%\|\bx_1+\bx_2\|^2 &\leq& (1+t)\|\bx_1\|^2+(1+ 1/t)\|\bx_2\|^2, \\
%\|\sum_{i=1}^m\bx_i\|^2 &\leq& m\sum_{i=1}^m\|\bx_i\|^2 .
 \end{array}
\end{eqnarray}
For notational simplicity, hereafter, we denote
\begin{eqnarray} 
 \label{decreasing-property-0}   
 \eqspace{1.5}
 \begin{array}{llllll}
\triangle\bx^{\tk}&:=&\bx^{\tk}-\bx^{\k },~&\triangle\blx_i^{k+1}&:=&\bx^{k+1}_i-\bx^{\tk},~~\\
\triangle\bx_i^{k+1}&:=&\bx_i^{k+1}-\bx_i^{k},~~~&
\bg_i^{k+1}&:=&\nabla f_i(\bx_i^{k+1}),
    \end{array}
 \end{eqnarray} 
 and let   $\bx^{k} \rightarrow \bx$ stand for $\lim_{k\rightarrow\infty} \bx^{k} = \bx$.  Before we prove all theorems, we claim some useful facts.
\begin{lemma} For any $i\in[m]$ and at $(k+1)$th iteration, let $k_i$ be the largest integer in $[-1,k]$ such that $i\in \S^{\tau_{k_i+1}}$. Then we have the following statements.
 \begin{itemize} 
 \item[i)] For every $i\in[m]$ and $k\geq 0$,
\begin{eqnarray}
\label{iceadmm-sub4-all}  
\bz^{\tk}_i = \bx_i^{k+1}  + \bze_i^{\tk}
\end{eqnarray}
where
  \begin{eqnarray} \label{def-zeta-ik}
\eqspace{1.5}
\bze_i^{\tau_{k+1}}:=\left\{
\begin{array}{lll}
 \bxi_i^{\tau_{k+1}},& i\in \S^{\tk},\\
 \bxi_i^{\tau_{k_i+1}},&i\notin \S^{\tk}, ~k_i \geq0,\\
 \bxi_i^{0},&i\notin \S^{\tk}, ~k_i=-1.
\end{array}\right.
\end{eqnarray}
%\item[ii)] For any  $i\in \S^{\tk}$, it has
%  \begin{eqnarray} \label{eps-k-k+1}
%\eqspace{1.5} 
%\begin{array}{lll}
%\frac{1}{ \mu_{i,k+1}} \leq \frac{1}{ \mu_{i,0}( \alpha_i-1)} (\frac{1}{ \alpha_i^{k}} -\frac{1}{ \alpha_i^{k+1}} ).
%\end{array} 
%\end{eqnarray} 
\item[iii)] Under Setup \ref{setup-D-1}, for any  $k\in\K$, it has
  \begin{eqnarray} \label{bd-zeta-ik}
\eqspace{1.5} 
\begin{array}{lll}
\E\varphi(\bze_i^{\tau_{k+1}})\leq  \phi_{i,k} -  \phi_{i,k+1}. 
\end{array} 
\end{eqnarray} 

\end{itemize}
\end{lemma}
\begin{proof} i) The definition of $k_i$ implies that client $i$ is not selected in all $\S^{\tau_{k_i+2}}, \S^{\tau_{k_i+3}} \cdots ,\S^{\tk}$, which by \eqref{iceadmm-sub5}  yields
 \begin{eqnarray}\label{par-invariance}
\eqspace{1.5}
\begin{array}{rcl}
( \bx^{\ell+1}_i,  \bz^{\tau_{\ell+1}}_i) \equiv (  \bx^{k_i+1}_i,   \bz^{\tau_{k_i+1}}_i), \forall \ell =  k_i,k_i+1, \cdots ,k.
\end{array} \end{eqnarray}  
For any client $i\in \S^{\tk}$, we have \eqref{iceadmm-sub4}. For any client $i\notin \S^{\tk}$, if ${k_i   }\geq0$,  then $( \bx^{k_i+1}_i,  \bz^{k_i+1}_i)$ also satisfies \eqref{iceadmm-sub4} due to $i\in \S^{\tau_{k_i+1}}$, which by condition \eqref{par-invariance} implies that 
 \begin{eqnarray*}  
~~\bz^{\tk}_i \overset{\eqref{par-invariance}}{=} \bz^{\tau_{k_i+1}}_i \overset{\eqref{iceadmm-sub4} }{=} \bx_i^{k_i+1}  + \bxi_i^{\tau_{k_i+1}} \overset{\eqref{par-invariance}}{=} \bx_i^{k+1}  + \bxi_i^{\tau_{k_i+1}}. 
\end{eqnarray*}
If $k_i=-1$, this means that is client $i$ has never been selected. Then by \eqref{par-invariance} and our initialization, we have  \begin{eqnarray*} 
\eqspace{1.25}
\begin{array}{lcl}\bz^{\tk}_i\overset{\eqref{par-invariance}}{=}\bz^{\tau_{k_i+1}}_i=\bz^{0}_i =\bx_i^{0} +\bxi_i^{0} \overset{\eqref{par-invariance}}{=} \bx_i^{k+1} +\bxi_i^{0}.\end{array} \end{eqnarray*} 
%ii) It follows from \eqref{iceadmm-sub2} that
%\begin{eqnarray*} 
%\eqspace{1.5} 
%\begin{array}{lll}
%\frac{1}{ \mu_{i,k+1}} \leq  \frac{1}{ \mu_{i,0} \alpha_i^{k+1}} 
%&=&  \frac{1}{  \alpha_i-1} (\frac{1}{ \mu_{i,0} \alpha_i^{k}} -\frac{1}{\mu_{i,0} \alpha_i^{k+1}} ).
%%&=&  \frac{1}{  \alpha_i-1} (\frac{1}{ \mu_{i,k}  } -\frac{1}{\mu_{i,k+1}  } ).
%\end{array} 
%\end{eqnarray*} 
ii) Under Setup \ref{setup-D-1}, we can observe that $\tau_{k+1} - \tau_{k_i+1} < 2s_0$ from \eqref{scheme-omega-2T}, 
which by $k\in\K$ implies $\tau_{k} =   \tau_{k+1}-1 < \tau_{k_i+1} +2s_0-1 \leq  \tau_{k_i}+  2s_0$, thereby leading to
\begin{eqnarray}\label{k-ki-2s0}
\eqspace{1.5}
\begin{array}{lcl}
 \tau_{k_i} > \tau_{k}  -2s_0   = \tau_{(k -2s_0k_0)}.
\end{array} \end{eqnarray} 
This indicates
\begin{eqnarray}\label{k-ki-2s1}
\eqspace{1.5}
\begin{array}{lcl}
 k_i  >  k  -2s_0k_0.
\end{array} \end{eqnarray} 
%where the last inequality is from $k >2s_0k_0$. The above condition indicates $k_i\geq0$, which by \eqref{def-zeta-ik} yields 
%  \begin{eqnarray} \label{def-zeta-ik-1}
%\eqspace{1.5}
%\bze_i^{\tau_{k+1}}:=\left\{
%\begin{array}{lll}
% \bxi_i^{\tau_{k+1}},& i\in\S^{\tk},\\
% \bxi_i^{\tau_{k_i+1}},&i\notin\S^{\tk}.
%\end{array}\right.
%\end{eqnarray}
We note that for any $\epsilon\sim{\rm Lap}(0,\nu)$, it holds
  \begin{eqnarray} \label{Ex-xi1-xi2}
\eqspace{1.5}
\begin{array}{lclcl}
\E  | \epsilon| = 4 \nu,\qquad  \E  \epsilon^2 = 16 \nu^2.\end{array} \end{eqnarray}
Now we denote three constants for notational simplicity,   
\begin{eqnarray}  \label{nu-decreasing-0}
\eqspace{1.5}
\begin{array}{lclcl}
\nu_k:=\frac{ \Delta_i^{\k}}{\varepsilon \mu_{i, k} },~~\nu_k^{\infty}:=\frac{ \Delta_i^{\infty}}{\varepsilon  \mu_{i, k} },~~u_k^{\infty}:=\frac{ \Delta_i^{\infty}\alpha_{i}^{2s_0k_0}}{ \varepsilon \mu_{i, 0}\alpha_{i}^k}.
\end{array} \end{eqnarray} 
To estimate $\E \varphi(\bze_i^{\tau_{k+1}})$, there are three cases due to \eqref{def-zeta-ik}. 

 \begin{itemize}[leftmargin=*]
 \item \underline{Case $i\in\S^{\tk}$.}   Since $ \bxi_i^{\tau_{k+1}}$ is sampled as (\ref{sample-noise}), it follows
% \begin{eqnarray} \label{Ex-xi-ik}
%\eqspace{1.5}
%\begin{array}{lclcl}
%\E \varphi(\bxi_i^{\tau_{k+1}})
%&=&\sum_{j=1}^n  (\lambda  \E |\epsilon_{ij}^{\tau_{k+1}}| +  \frac{\eta}{2} \E |\epsilon_{ij}^{\tau_{k+1}}|^2)\\
%& \overset{\eqref{Ex-xi1-xi2}}{=}& \sum_{j=1}^n (\frac{4\lambda\Delta_i^k}{\varepsilon\mu_{i,\k }} + \frac{8\eta(\Delta_i^k)^2}{(\varepsilon\mu_{i,\k })^2})\\
%& =&  \frac{4n\lambda\Delta_i^k}{\varepsilon\mu_{i,\k }} + \frac{8n\eta(\Delta_i^k)^2}{(\varepsilon\mu_{i,\k })^2}\\
%& \overset{\eqref{scale-para}}{ \leq}&   \frac{4\lambda\Delta_i^{\infty}}{\varepsilon\mu_{i,\k }} + \frac{8\eta(\Delta_i^\infty)^2}{n(\varepsilon\mu_{i,\k })^2}.\end{array} \end{eqnarray} 
 \begin{eqnarray} \label{Ex-xi-ik}
\eqspace{1.5}
\begin{array}{lclcl}
\E \varphi(\bxi_i^{\tau_{k+1}})
&=&\sum_{j=1}^n  (\lambda  \E |\epsilon_{ij}^{\tau_{k+1}}| +  \frac{\eta}{2} \E |\epsilon_{ij}^{\tau_{k+1}}|^2)\\
& \overset{\eqref{Ex-xi1-xi2}}{=}& 
\sum_{j=1}^n ( 4\lambda\nu_{k+1} + 8 \eta \nu_{k+1}^2)\\
& =&  4n\lambda\nu_{k+1} + 8n \eta \nu_{k+1}^2\\
& \overset{\eqref{scale-para}}{ \leq}&   4n\lambda\nu_{k+1}^\infty + 8n \eta (\nu_{k+1}^\infty)^2.\end{array} \end{eqnarray} 
By \eqref{iceadmm-sub2}, $\mu_{i,k+1} \geq  \mu_{i,0}\alpha_i^{k+1}$ and $\alpha_i>1$, there is  
\begin{eqnarray} \label{nu-decreasing}
\eqspace{1.5}
\begin{array}{lclcl}
 \nu_{k+1}^{\infty}\overset{\eqref{nu-decreasing-0}}{=}\frac{ \Delta_i^{\infty}}{ \varepsilon \mu_{i,k+1}  } =\frac{ \Delta_i^{\infty}}{ \varepsilon \mu_{i,0}  \alpha_{i}^{ k+1} } \overset{\eqref{nu-decreasing-0}}{\leq} u_{k+1}^{\infty}.
\end{array} \end{eqnarray} 
The above two facts  result in
\begin{eqnarray} \label{Ex-xi-ik-1}
\eqspace{1.75}
\begin{array}{lcl}
\E \varphi(\bze_i^{\tau_{k+1}}) &\overset{\eqref{def-zeta-ik}}{ =}& \E \varphi(\bxi_i^{\tau_{k+1}})\\
 &  \overset{\eqref{Ex-xi-ik}}{ \leq}&   4n\lambda\nu_{k+1}^\infty + 8n \eta (\nu_{k+1}^\infty)^2 \\
&  \overset{\eqref{nu-decreasing}}{ \leq}&   4n\lambda u_{k+1}^{\infty} + 8n \eta (u_{k+1}^{\infty})^2.
%&=& \frac{4\lambda\Delta_i^{\infty}\alpha_{i}^{2s_0}}{\varepsilon \mu_{i,0}  \alpha_{i}^{\k}} + \frac{8n\eta(\Delta_i^\infty\alpha_{i}^{2s_0})^2}{(\varepsilon \mu_{i,0})^2  \alpha_{i}^{2\k}}. 
\end{array} \end{eqnarray} 

\item  \underline{Case $i\notin\S^{\tk}, k_i\geq0$.}  We can check that
\begin{eqnarray} \label{nu-decreasing-1}
\eqspace{1.75}
\begin{array}{lclcl}
 \nu_{k_i+1}^{\infty} &\overset{\eqref{nu-decreasing-0}}{=}& \frac{ \Delta_i^{\infty}}{\varepsilon\mu_{i, k_i+1}}=\frac{ \Delta_i^{\infty}}{ \varepsilon \mu_{i,0}  \alpha_{i}^{k_i+1}}\\
 & \overset{\eqref{k-ki-2s1}}{\leq}&  \frac{ \Delta_i^{\infty}}{ \varepsilon \mu_{i,0}  \alpha_{i}^{k+1-2s_0k_0}}\overset{\eqref{nu-decreasing-0}}{=} u_{k+1}^{\infty}.
\end{array} \end{eqnarray} 
Similar to prove \eqref{Ex-xi-ik}, we can obtain
\begin{eqnarray*} 
\eqspace{1.75}
\begin{array}{lcl}
\E \varphi(\bze_i^{\tau_{k+1}}) &\overset{\eqref{def-zeta-ik}}{ =}&\E \varphi(\bxi_i^{\tau_{k_i+1}})\\
& \overset{\eqref{Ex-xi-ik}}{\leq}& 4n\lambda\nu_{k_i+1}^\infty + 8n \eta (\nu_{k_i+1}^\infty)^2 \\
&  \overset{\eqref{nu-decreasing-1}}{ \leq}&   4n\lambda u_{k+1}^{\infty} + 8n \eta (u_{k+1}^{\infty})^2.
\end{array} \end{eqnarray*} 
\item  \underline{Case $i\notin\S^{\tk}, k_i=-1$.}  Under such a case, condition \eqref{k-ki-2s0} means that $k<2s_0k_0$, leading to $k+1\leq 2s_0k_0$. Moreover,
\begin{eqnarray*} 
\eqspace{1.75}
\begin{array}{lcl}
 \nu_{0}^\infty  = \frac{ \Delta_i^{\infty}}{\varepsilon \mu_{i,0}}   \leq   \frac{ \Delta_i^{\infty}\alpha_{i}^{2s_0k_0}}{ \varepsilon \mu_{i,0}  \alpha_{i}^{k+1}} =u_{k+1}^{\infty}.
\end{array} \end{eqnarray*} 
 Then the same reasoning enables us to derive
\begin{eqnarray*} 
\eqspace{1.75}
\begin{array}{lcl}
\E \varphi(\bze_i^{\tau_{k+1}}) &\overset{\eqref{def-zeta-ik}}{ =}&\E \varphi(\bxi_i^{0})\\
 & \overset{\eqref{Ex-xi-ik}}{\leq}& 4n\lambda\nu_{0}^\infty + 8n \eta (\nu_{0}^\infty)^2 \\
   &\leq& 4n\lambda u_{k+1}^{\infty} + 8n \eta (u_{k+1}^{\infty})^2.
\end{array} \end{eqnarray*} 
\end{itemize}
Hence, the above three cases and \eqref{def-zeta-ik} allows us to obtain
 \begin{eqnarray*} 
\eqspace{1.75} 
\begin{array}{lll}
\E\varphi(\bze_i^{\tau_{k+1}}) &\leq& 4n\lambda u_{k+1}^{\infty} + 8n \eta (u_{k+1}^{\infty})^2\\
&=&\frac{4n\lambda\Delta_i^{\infty}\alpha_{i}^{2s_0k_0}}{\varepsilon \mu_{i,0}  \alpha_{i}^{ k+1}} + \frac{8n\eta(\Delta_i^\infty\alpha_{i}^{2s_0k_0})^2}{(\varepsilon \mu_{i,0})^2  \alpha_{i}^{2( k+1)}}\\
& =& \frac{4n\lambda\Delta_i^{\infty} \alpha_{i}^{2s_0k_0}}{\varepsilon \mu_{i,0} }\Big(\frac{1}{\alpha_i-1}\frac{1}{ \alpha_{i}^{k}} -\frac{1}{\alpha_i-1}\frac{1}{ \alpha_{i}^{k+1 }} \Big)\\
&+& \frac{8n\eta(\Delta_i^\infty\alpha_{i}^{2s_0k_0})^2}{(\varepsilon \mu_{i,0})^2  }\Big(\frac{1}{\alpha_i^2-1}\frac{1}{ \alpha_{i}^{2k  }} -\frac{1}{\alpha_i^2-1}\frac{1}{ \alpha_{i}^{2(k+1)}} \Big)\\
& =& \frac{4n\lambda\Delta_i^{\infty} \alpha_{i}^{2s_0k_0}}{\varepsilon \mu_{i,0}(\alpha_i-1)  }\Big( \frac{1}{ \alpha_{i}^{k}} - \frac{1}{ \alpha_{i}^{k+1 }} \Big)\\
&+& \frac{8n\eta(\Delta_i^\infty\alpha_{i}^{2s_0k_0})^2}{(\varepsilon \mu_{i,0})^2 (\alpha_i^2-1)  }\Big( \frac{1}{ \alpha_{i}^{2k  }} - \frac{1}{ \alpha_{i}^{2(k+1)}} \Big)\\
%& =& \frac{4n\lambda\Delta_i^{\infty} \alpha_{i}^{2s_0+1}}{\varepsilon \mu_{i,0} (\alpha_i-1)}\Big( \frac{1}{ \alpha_{i}^{\tau_{ k}}} - \frac{1}{ \alpha_{i}^{\tk}} \Big)\\
%&+& \frac{8n\eta(\Delta_i^\infty\alpha_{i}^{2s_0+1})^2}{(\varepsilon \mu_{i,0})^2 (\alpha_i^2-1) }\Big( \frac{1}{ \alpha_{i}^{2\tau_{ k}}} - \frac{1}{ \alpha_{i}^{2\tk}} \Big)\\
& =&   \phi_{i,k} -  \phi_{i,k+1}.
\end{array} 
\end{eqnarray*} 
 displaying the result.
  \end{proof}
  
   \subsection{Proof of Lemma \ref{lemma-decreasing-1}}   
\begin{proof} First, we can conclude that, if $f$ is a strongly convex function with a constant $r>0$, then for any $\bx$ it holds
  \begin{eqnarray} 
\label{strong-convex}  
\eqspace{1.5}
\begin{array}{lll}
 f(\bx) &\geq&  f(\bv) + \langle \nabla f(\bv), \bx-\bv \rangle  + \frac{r}{2}\|\bx-\bv\|^2\\
 &=&  f(\bv)  + \frac{r}{2}\|\bx-\bv\|^2,
    \end{array} 
 \end{eqnarray} 
 where $\bv={\rm argmin}_{\bx} f(\bx).$ To prove the resuls, we  aim to estimate the following item, 
  \begin{eqnarray} 
\label{three-cases}  
\begin{array}{l}
 F(\bx^{\tk},W^{k+1})-F(\bx^{\k },W^{k}) =: q_1^k+q_2^k,
    \end{array} 
 \end{eqnarray} 
where
   \begin{eqnarray}  \label{three-cases-sub}  
\eqspace{1.5} 
 \begin{array}{llllll}
q_1^k&:=&F(\bx^{\tk},W^{k})-F(\bx^{\k },W^{k}),\\
q_2^k&:=& F(\bx^{\tk},W^{k+1})-F(\bx^{\tk},W^{k}),
    \end{array} 
 \end{eqnarray} 
{\underline{Estimate  $q_1^k$.}} We first focus on case $k\in\K$. 
%The optimality condition of \eqref{iceadmm-sub1} is  
%\begin{eqnarray*}\eqspace{1.5} 
% \begin{array}{llll}
%\sum_{i=1}^m {\bf u}_i^{k+1}=0, ~~\text{where}~{\bf u }_i^{k+1}:= \lambda \widetilde\bpi_i^{k+1}+\eta (\bx^{\tk}-\bz^{\k }_i)),
%\end{array}
%\end{eqnarray*}
%for $\widetilde\bpi_i^{k+1} \in \sign(\bx^{\tk}-\bz^{\k }_i) $. 
It is noted that $\varphi(\bz^{\k }_i-\cdot)$ is strongly convex with a constant $\eta$ and thus 
\begin{eqnarray}\label{strong-cconvexity}
\eqspace{1.5} 
 \begin{array}{llll}
&&\sum_{i=1}^m  (\varphi(\bz^{\k }_i-\bx^{\tk})-\varphi(\bz^{\k }_i-\bx^{\k}))\\
&\overset{\eqref{strong-convex}}{\leq} &    \sum_{i=1}^m  -\frac{\eta}{2}\|\triangle  \bx^{\tk}\|^2.
\end{array}
\end{eqnarray} 
Moreover, by $\bz^{\k }_i = \bx_i^k+\bze_i^{\k }$ from \eqref{iceadmm-sub4-all},  we have
 \begin{eqnarray*}   
\eqspace{1.5} 
 \begin{array}{lrl} 
t_1&:=& \|\bx_i^k-\bx^{\tk}\|_1- \|\bx_i^k-\bx^{\k }\|_1  \\
 &\leq &   \|\bz^{\k }_i- \bx^{\tk} \|_1 + \|  \bze_i^{\k } \|_1 - \|\bx_i^k-\bx^{\k }\|_1  \\
 &\leq & \|\bz^{\k }_i- \bx^{\tk} \|_1 - \|\bz^{\k }_i- \bx^{\k }\|_1 + 2\|  \bze_i^{\k } \|_1
    \end{array} 
 \end{eqnarray*}
 and
  \begin{eqnarray*}   
\eqspace{1.5} 
 \begin{array}{lrl} 
t_2&:=& \|\bx_i^k-\bx^{\tk}\|^2- \|\bx_i^k-\bx^{\k }\|^2  \\ 
 &=& \|\bz^{\k }_i -  \bze_i^{\k } -\bx^{\tk}   \|^2 - \| \bz^{\k }_i - \bze_i^{\k } -\bx^{\k }  \|^2\\
 &=& \|\bz^{\k }_i  -\bx^{\tk}  \|^2 - \|\bz^{\k }_i  -\bx^{\k } \|^2 +2 \langle \triangle  \bx^{\tk}, \bze_i^{\k } \rangle\\
 &\overset{\eqref{two-vecs}}{\leq}& \|\bz^{\k }_i  -\bx^{\tk}  \|^2 - \|\bz^{\k }_i  -\bx^{\k } \|^2 \\ 
 &+& \frac{1}{2}\|\triangle  \bx^{\tk}\|^2+ 2\| \bze_i^{\k }\|^2.
    \end{array} 
 \end{eqnarray*}
% Then the above two facts enable us to obtain  
%   \begin{eqnarray*}   
%\eqspace{1.5} 
% \begin{array}{lll} 
% &&\varphi(\bx_i^k-\bx^{\tk})- \varphi(\bx_i^k-\bx^{k })  = \lambda c_i +  (\eta/2) c_i \\ 
% &\leq& \varphi(\bz^{k }_i   -\bx^{\tk}  ) - \varphi(\bz^{k }_i  -\bx^{k } ) +  2\lambda \|  \bze_i^{k } \|_1 + \eta \langle \triangle  \bx^{\tk}, \bze_i^{k } \rangle.
%    \end{array} 
% \end{eqnarray*}
 The above two facts allow  us to derive that
 \begin{eqnarray}\label{gap-3-1*}    
\eqspace{1.75} 
 \begin{array}{lcl} 
q_1^k    &\overset{\eqref{FL-EPM}}{=}&    \sum_{i=1}^{m}    \varphi(\bx^{k }_i-\bx^{\tk}) - \varphi(\bx^{k }_i-\bx^{\k})\\
  &=&    \sum_{i=1}^{m}   ( \lambda t_1 +   \eta t_2/2) \\
 &=&    \sum_{i=1}^{m}    (\varphi(\bz^{\k }_i   -\bx^{\tk}  ) - \varphi(\bz^{\k }_i  -\bx^{k } ))    \\
 &+& \sum_{i=1}^{m}   ( 2\lambda \|  \bze_i^{\k } \|_1 + \frac{\eta}{4}\|\triangle  \bx^{\tk}\|^2+ \eta\| \bze_i^{\k }\|^2) \\ 
   &\overset{\eqref{strong-cconvexity}}{\leq}&     \sum_{i=1}^{m}   ( 2\lambda \|  \bze_i^{\k } \|_1 + \eta \|  \bze_i^{\k } \|^2 -\frac{\eta}{4}\|\triangle  \bx^{\tk}\|^2) \\ 
    &\overset{\eqref{Elastic}}{=}&     \sum_{i=1}^{m}   ( 2  \varphi( \bze_i^{\k } ) -\frac{\eta}{4}\|\triangle  \bx^{\tk}\|^2). \\ 
 %    &\overset{\eqref{bd-zeta-ik}}{\leq}&     \sum_{i=1}^{m}   (  \phi_i^{k-1}-\phi_i^k -\frac{\eta}{4}\|\triangle  \bx^{\tk}\|^2).
    \end{array} 
 \end{eqnarray} 
 If $k\notin\K$, then $\tk=\k $ and hence $q_1^k=0$, which means the last inequality in the above condition is still valid. 
 
\noindent{\underline{Estimate  $q_2^k$.}} We first consider any client $i\in\S^{\tk}$. Since $\bx_i^{k+1}$ is an optimal solution to \eqref{iceadmm-sub2-prob-1}, then it satisfies the following optimality condition,
\begin{eqnarray}\label{iceadmm-sub2-g-pi-u}
 \eqspace{1.5}
\begin{array}{lll} 
 \bg_i^{\tk} + \mu_{i,k+1} \triangle \bx_i^{k+1}+ \lambda \bu_i^{k+1}  + \eta \triangle \blx_i^{k+1} =0,
\end{array} 
\end{eqnarray}
where $\bu_i^{k+1}\in\sgn (\triangle \blx_i^{k+1})$. Next we estimate several terms. The gradient Lipschitz continuity of $f_i$ gives rise to
\begin{eqnarray*} 
 \eqspace{1.5}
\begin{array}{lll} 
t_3 :=  f_i( \bx_i^{k+1}) - f_i(\bx_i^k)    \overset{\eqref{H-Lip-continuity-fxy}}{\leq} & \langle \bg_i^{k}, \triangle \bx_i^{k+1} \rangle + \frac{r_i}{2}\|\triangle\bx_i^{k+1}\|^2.
\end{array} 
\end{eqnarray*}
Since $\bu_i^{k+1}\in\sgn (\triangle \blx_i^{k+1})$, it holds
\begin{eqnarray*} 
 \eqspace{1.5}
\begin{array}{lll} 
&&\langle   \bu_i^{k+1}, -\triangle\bx_i^{k+1}\rangle + \| \triangle \blx_i^{k+1}\|_1-\|\bx_i^{k}-\bx^{\tk}\|_1 \\
&=&   \langle   \bu_i^{k+1}, \triangle \blx_i^{k+1}-\triangle\bx_i^{k+1}\rangle  -\|\bx_i^{k}-\bx^{\tk}\|_1 \\ 
&=&   \langle  \bu_i^{k+1}, \bx_i^{k}-\bx^{\tk} \rangle  -\|\bx_i^{k}-\bx^{\tk}\|_1 \\
&\leq& 0,
\end{array} 
\end{eqnarray*}
where the last inequality  is due to $\bu_i^{k+1}\in\sgn (\triangle \blx_i^{k+1})$, which immediately results in
\begin{eqnarray*} 
 \eqspace{1.5}
\begin{array}{lll} 
 t_4:=  \| \triangle \blx_i^{k+1}\|_1-\|\bx_i^{k}-\bx^{\tk}\|_1 
 \leq  \langle   \bu_i^{k+1}, \triangle\bx_i^{k+1}\rangle.
\end{array} 
\end{eqnarray*}
Using this fact, we have
\begin{eqnarray*} 
 \eqspace{1.5}
\begin{array}{lll} 
t_5&:=& \varphi(\bx_i^{k+1}-\bx^{\tk}) - \varphi(\bx_i^k-\bx^{\tk})  \\
&=&\lambda t_4 + \frac{\eta}{2} (\|\bx_i^{k+1}-\bx^{\tk}\|^2-\|\bx_i^{k}-\bx^{\tk}\|^2)\\ 
&=&\lambda t_4 + \frac{\eta}{2} ( \| \triangle \blx_i^{k+1}\|^2-\|\triangle  \blx_i^{k+1}-\triangle\bx_i^{k+1}\|^2)\\
&\leq &  \langle \lambda \bu_i^{k+1} +  \eta \triangle  \blx_i^{k+1}, \triangle\bx_i^{k+1}\rangle  - \frac{\eta}{2} \| \triangle\bx_i^{k+1}\|^2. 
\end{array} 
\end{eqnarray*}
Moreover, direct calculations can verify that
\begin{eqnarray}\label{p2k-fact2-g-g}   
\eqspace{1.75} 
 \begin{array}{lcl}
 && \langle \bg_i^{k}- \bg_i^{\tk}, \triangle\bx_i^{k+1}\rangle  \\
 &\overset{\eqref{two-vecs}}{\leq}& \frac{1}{2\mu_{i,k+1}}\|\bg_i^{k}-\bg_i^{\tk}\|^2 + \frac{\mu_{i,k+1}}{2}\| \triangle\bx_i^{k+1}\|^2\\
  & {\leq}& \frac{r_i^2}{2\mu_{i,k+1}}\|\bx_i^{k}-\bx^{\tk}\|^2 + \frac{\mu_{i,k+1}}{2}\| \triangle\bx_i^{k+1}\|^2\\
%  & {\leq}& \frac{r_i^2}{\mu_{i,k+1}}(\|\triangle \blx_i^{k}\|^2+\|\triangle\bx^{\tk}\|^2) + \frac{\mu_{i,k+1}}{2}\| \triangle\bx_i^{k+1}\|^2\\
  & \overset{\eqref{iceadmm-sub2}}{\leq}& \frac{r_i^2}{2\mu_{i,0}c_i\alpha_i^{k+1} } + \frac{\mu_{i,k+1}}{2}\| \triangle\bx_i^{k+1}\|^2.
    \end{array} 
 \end{eqnarray}
Combining the above facts,  we have
\begin{eqnarray}\label{p2k-fact2}   
\eqspace{1.75} 
 \begin{array}{lcl}
 &&F_i(\bx^{\tk},\bx_i^{k+1})-F_i(\bx^{\tk},\bx_i^{k}) \overset{\eqref{FL-EPM}}{=} t_3+t_5\\
 &\leq& \langle \bg_i^{k}{+} \lambda \bu_i^{k+1} {+}  \eta \triangle  \blx_i^{k+1}, \triangle\bx_i^{k+1}\rangle  {+} \frac{r_i-\eta}{2} \| \triangle\bx_i^{k+1}\|^2\\ 
  &\overset{\eqref{iceadmm-sub2-g-pi-u}}{=}& \langle \bg_i^{k}-\bg_i^{\tk}, \triangle\bx_i^{k+1}\rangle  + \frac{r_i-\eta-2\mu_{i,k+1 }}{2} \| \triangle\bx_i^{k+1}\|^2\\ 
% &\overset{\eqref{two-vecs}}{\leq}&  \frac{1}{2\mu_{i,k+1 }} \|\widetilde{\bg}_i^{k+1}\|^2  + \frac{\gamma_i}{2} \| \triangle\bx_i^{k+1}\|^2\\
  &\overset{\eqref{p2k-fact2-g-g}  }{\leq}&  \frac{r_i^2}{2\mu_{i,0}c_i\alpha_i^{k+1} } + \frac{r_i-\eta- \mu_{i,k+1 }}{2} \| \triangle\bx_i^{k+1}\|^2.
%  &\overset{\eqref{eps-k-k+1}}{\leq}&  \frac{v_i}{2\eta(1-v_i)} (\epsilon_{i,k}  -\epsilon_{i,k+1})   - \frac{2\mu_{i,\k }-r_i}{2} \| \triangle\bx_i^{k+1}\|^2.
    \end{array} 
 \end{eqnarray}
For any client $i{\notin}\S^{\tk}$, it follows from \eqref{iceadmm-sub5} that $\bx_i^{k+1}{=}\bx_i^k$, which means the above condition is still valid. Overall, condition \eqref{p2k-fact2}  is true for any $i\in[m]$, thereby giving rise to
 \begin{eqnarray*} 
 \arraycolsep=1.5pt\def\arraystretch{1.5}
 \begin{array}{lcl}
 q_2^k & =&   \sum_{i=1}^{m}    (F_i(\bx^{\tk},\bx_i^{k+1})-F_i(\bx^{\tk},\bx_i^{k}) )  \\    
 & {\leq}& \sum_{i=1}^{m}   (  \frac{r_i^2}{2\mu_{i,0}c_i\alpha_i^{k+1} }   + \frac{r_i-\eta- \mu_{i,k+1 }}{2} \| \triangle\bx_i^{k+1}\|^2).
    \end{array} 
 \end{eqnarray*}
 Then, this condition, \eqref{gap-3-1*}, and \eqref{three-cases} allow  us to show  \eqref{decreasing-property-2}.
 
 We finally prove \eqref{decreasing-property-1}. For any client $i\in\S^{\tk}$,
\begin{eqnarray}\label{p2k-fact2-1}   
\eqspace{1.75} 
 \begin{array}{lcl}
 &&F_i(\bx^{\tk},\bx_i^{k+1})-F_i(\bx^{\tk},\bx_i^{k})\\ 
  &\overset{\eqref{p2k-fact2}}{\leq}&  \frac{r_i^2}{2\mu_{i,0}c_i\alpha_i^{k+1} }  + \frac{r_i-\eta- \mu_{i,k+1 }}{2} \| \triangle\bx_i^{k+1}\|^2\\
  &=&  \frac{r_i^2}{2\mu_{i,0}c_i (\alpha_i-1)} (\frac{1}{ \alpha_i^{k} } -\frac{1}{\alpha_i^{k+1}} ) +    \frac{r_i-\eta- \mu_{i,k+1 }}{2} \| \triangle\bx_i^{k+1}\|^2.
    \end{array} 
 \end{eqnarray}
Since $\bx_i^{k+1}=\bx_i^k$ for any client $i\notin\S^{\tk}$,  we can conclude that condition \eqref{p2k-fact2-1} is true for any $i\in[m]$, thereby yielding
 \begin{eqnarray*} 
 \arraycolsep=0.5pt\def\arraystretch{1.75}
 \begin{array}{lcl}
 q_2^k & =&   \sum_{i=1}^{m}    (F_i(\bx^{\tk},\bx_i^{k+1})-F_i(\bx^{\tk},\bx_i^{k}) )  \\    
 & {\leq}& \sum_{i=1}^{m}  \frac{r_i^2}{2\mu_{i,0}c_i (\alpha_i-1)} (\frac{1}{ \alpha_i^{k} } -\frac{1}{\alpha_i^{k+1}} )\\
 &  +&  \sum_{i=1}^{m}   \frac{r_i-\eta- \mu_{i,k+1 }}{2} \| \triangle\bx_i^{k+1}\|^2.
    \end{array} 
 \end{eqnarray*}
Combining the above condition, \eqref{gap-3-1*}, and \eqref{three-cases}, we can claim
  \begin{eqnarray*} \eqspace{1.75}
  \begin{array}{lcl}
&&\E_{\bxi} F(\bx^{\tk},W^{k+1})- \E_{\bxi} F(\bx^{\k },W^{k})\\ 
 &\leq& 
      \sum_{i=1}^{m}( \frac{r_i^2}{2\mu_{i,0}c_i (\alpha_i-1)} (\frac{1}{ \alpha_i^{k} } -\frac{1}{\alpha_i^{k+1}} )   +2  \E_{\bxi} \varphi( \bze_i^{\k } ) )\\
 &+&  \sum_{i=1}^{m}(  \frac{r_i-\eta- \mu_{i,k+1 }}{2} \E_{\bxi}\| \triangle\bx_i^{k+1}\|^2  -\frac{\eta}{4} \E_{\bxi} \|\triangle  \bx^{\tk}\|^2   )\\
&\overset{\eqref{bd-zeta-ik}}{\leq}& 
      \sum_{i=1}^{m}(\frac{r_i^2}{2\mu_{i,0}c_i (\alpha_i-1)} (\frac{1}{ \alpha_i^{k} } -\frac{1}{\alpha_i^{k+1}} )  +  2\phi_{i,k-1} -  2\phi_{i,k} )\\
 &+&   \sum_{i=1}^{m}(  \frac{r_i-\eta- \mu_{i,k+1 }}{2} \E_{\bxi}\| \triangle\bx_i^{k+1}\|^2  -\frac{\eta}{4}  \E_{\bxi} \|\triangle  \bx^{\tk}\|^2   ),  
    \end{array}    
    \end{eqnarray*}  
showing \eqref{decreasing-property-1} after the simple manipulation. 
\end{proof}
 
\subsection*{Proof of  Theorem \ref{global-obj-convergence-inexact}}   

 \begin{proof} i) Since $f_i$ is bounded from below and
 \begin{eqnarray*}    
 \eqspace{1.75}
 \begin{array}{lll}
t_{i,k}&:=& \frac{r_i^2}{2\mu_{i,0}c_i (\alpha_i-1)\alpha_i^{k}  }   +  2\phi_{i,k-1}\geq 0,\\
\L^{k}&=& \E_{\bxi} F(\bx^{\k},W^{k})+ \sum_{i=1}^m t_{i,k},
%&=& \E_{\bxi} \sum_{i=1}^m (f_i(\bx_i^k) +  \varphi(\bx_i^k-\bx^{\k}) + t_{i,k}), 
    \end{array}  \end{eqnarray*}  
 sequence $ \{\L^{k}\}$ is bounded from below. We note that under Setup \ref{setup-D-1}, for each $2s_0$ consecutive iterations, every client $i$ is selected to update their parameters at least once. Therefore,   $\mu_{i,k+1}$ in the form of \eqref{iceadmm-sub2} is increasing for every $i\in[m]$. This means there must be a $k'>0$ such that for all $k\geq k'$, it holds  $\mu_{i,k'+1}>r_i-\eta$, which allows us to derive that
\begin{eqnarray}    \label{gap-LK1-LK}
 \eqspace{1.5}
 \begin{array}{lll}
   \L^{k+1} - \L^{k}
  &\overset{\eqref{decreasing-property-1} }{\leq} &      \sum_{i=1}^{m} \frac{ r_i-\eta-\mu_{i,k+1}}{2}\E_{\bxi}  \|\triangle\bx_i^{k+1}\|^2 \\
  &-&  \sum_{i=1}^{m}  \frac{\eta}{4}  \E_{\bxi} \|\triangle  \bx^{\tk}\|^2  \\
    &\leq &  \sum_{i=1}^{m} \frac{ r_i-\eta-\mu_{i,k'+1}}{2}\E _{\bxi} \|\triangle\bx_i^{k+1}\|^2 \\
  &-&  \sum_{i=1}^{m}  \frac{\eta}{4} \E_{\bxi} \|\triangle  \bx^{\tk}\|^2  \\
  &\leq&  0,
    \end{array}  
 \end{eqnarray} 
for all $k\geq k'$. Hence, sequence $\{ \L^k\}$ converges. Suppose that $ \Delta_i^{\infty}  \alpha_{i}^{-k} { \to }\infty $, then   $\phi_{i,k}{\to} \infty$ by \eqref{constants} and so is $t_{i,k}$, yielding
\begin{eqnarray*}    
 \eqspace{1.5}
 \begin{array}{lll}
\infty > \L^1 \geq \lim_{k\to\infty} \L^{k} = \infty, 
    \end{array}  
 \end{eqnarray*} 
 which is a contradiction. Hence,  $ \Delta_i^{\infty}  \alpha_{i}^{-k} $ is bounded, and so is $t_{i,k}$. Again by \eqref{constants},  it follows $t_{i,k+1} \leq t_{i,k} /\alpha_i$ and thus
\begin{eqnarray*}    
 \eqspace{1.5}
 \begin{array}{lll}
 \frac{t_{i,k+1}}{t_{i,1}}= \Pi_{\ell=1}^k \frac{t_{i,\ell+1}}{t_{i,\ell}} \leq\frac{1}{\alpha_i^k},
    \end{array}  
 \end{eqnarray*}
 which results in $t_{i,k}\to 0$ and hence
 \begin{eqnarray*}    
 \eqspace{1.75}
 \begin{array}{lll}
\lim\limits_{k\to\infty}\L^{k}&=& \lim\limits_{k\to\infty}(\E_{\bxi} F(\bx^{\k},W^{k})+ \sum_{i=1}^m t_{i,k})\\
&=& \lim\limits_{k\to\infty} \E_{\bxi} F(\bx^{\k},W^{k}). 
    \end{array}  \end{eqnarray*}  
 
ii) Taking the limit on both sides of inequalities \eqref{gap-LK1-LK} yields $\E_{\bxi} \|\triangle\bx_i^{k+1}\|^2\to 0$ and $\E_{\bxi} \|\triangle  \bx^{\tk}\|^2\to 0$ as $k\to\infty$. 
% 
% iii)  Let $(\bx^\infty,W^\infty)$  be an accumulating point, namely, there is a subsequence $\{(\bx^{\k }, W^k):k\in \J\}$, where $\J\subseteq\{0,1,2,3,\cdots\}$, such that $(\bx^{\k }, W^k)\to (\bx^\infty,W^\infty)$ as $k(\in\J)\to \infty$.
% 
% This completes the whole proof.
\end{proof}


\bibliographystyle{IEEEtran}
\bibliography{ref}

\newpage

% your own custom macro containing the $\backslash${\tt{includegraphics}} command to make things
%\bf{If you include a photo:}\vspace{-33pt}
%\begin{IEEEbiography}[{\includegraphics[width=1in,height=1.25in,clip,keepaspectratio]{fig1}}]{Michael Shell}
%Use $\backslash${\tt{begin\{IEEEbiography\}}} and then for the 1st argument use $\backslash${\tt{includegraphics}} to declare and link the author photo.
%Use the author name as the 3rd argument followed by the biography text.
%\end{IEEEbiography}

\vfill

\end{document}